\documentclass[11pt]{article}

\usepackage{microtype,graphicx,subfigure,booktabs}
\usepackage{hyperref}
\usepackage[margin=1.1in]{geometry}

\usepackage{xfrac}
\usepackage{xcolor}
\usepackage{latex-macros}
\usepackage{mdframed}
\usepackage{algorithm,algorithmic}
\usepackage{enumitem}
\usepackage[labelfont=bf]{caption}
\usepackage{pifont}% http://ctan.org/pkg/pifont

\newcommand{\Igood}{\cI_{\textrm{good}}}
\newcommand{\supp}{\mathrm{supp}}

\usepackage[capitalize]{cleveref}
\usepackage{natbib}

\theoremstyle{plain}
\newtheorem{theorem}{Theorem}[section]
\newtheorem{proposition}[theorem]{Proposition}
\newtheorem{lemma}[theorem]{Lemma}

\theoremstyle{definition}
\newtheorem{definition}[theorem]{Definition}
\newtheorem{assumption}[theorem]{Assumption}

\theoremstyle{remark}
\newtheorem{remark}[theorem]{Remark}

\newcommand{\epspriv}{\varepsilon_{\mathsf{priv}}}
\newcommand{\delpriv}{\delta_{\mathsf{priv}}}

\title{Robust and differentially private stochastic linear bandits}
\author{%
  Vasileios Charisopoulos\thanks{Part of this work was completed while the author was with Google.}\\
  Operations Research \& Information Engineering\\
  Cornell University\\
  \texttt{vc333@cornell.edu}\\
\and
  Hossein Esfandiari\\
  Google Research\\
  \texttt{esfandiari@google.com}\\
\and
  Vahab Mirrokni\\
  Google Research\\
  \texttt{mirrokni@google.com}
}

\crefname{theorem}{Theorem}{Theorems}
\crefname{claim}{Claim}{Claims}
\crefname{assumption}{Assumption}{Assumptions}
\crefname{lemma}{Lemma}{Lemmas}
\Crefname{lemma}{Lemma}{Lemmas}
\crefname{prop}{Proposition}{Propositions}
\crefname{corollary}{Corollary}{Corollaries}

\begin{document}

% \icmltitle{Robust and private stochastic linear bandits}]

% \onecolumn

% \icmlcorrespondingauthor{~}
% \printAffiliationsAndNotice{}
\maketitle

\begin{abstract}
  In this paper, we study the stochastic linear bandit problem under the additional
  requirements of \emph{differential privacy}, \emph{robustness} and \emph{batched observations}.
  In particular, we assume an adversary randomly chooses a constant
  fraction of the observed rewards in each batch, replacing them with arbitrary numbers.
  We present differentially private and robust variants of the arm elimination algorithm
  using logarithmic batch queries under two privacy models and provide regret
  bounds in both settings. In the first model,
  every reward in each round is reported by a potentially different client, which
  reduces to standard local differential privacy (LDP). In the second model,
  every action is ``owned'' by a different client, who may aggregate the rewards
  over multiple queries and privatize the aggregate response instead. To the best of
  our knowledge, our algorithms are the first simultaneously providing differential
  privacy and adversarial robustness in the stochastic linear bandits problem.
\end{abstract}

\section{Introduction}
Bandits model is a popular formulation for online learning, wherein a
learner interacts with her environment by choosing a sequence of actions, each
of which presents a \emph{reward} to the learner, from an available (potentially
infinite) set of actions. The goal of the learner is to minimize her \emph{regret},
defined as the difference between the rewards obtained by the chosen
sequence of actions and the best possible action in hindsight. To achieve this,
the learner must balance between \emph{exploration} (choosing actions that reveal
information about the action set) and \emph{exploitation} (repeating actions
that offered the highest rewards in previous rounds).

In theory, deciding the next action sequentially is easiest. However, there are several obstacles to overcome when it comes to practice.
The first obstacle is that the rewards in bandit algorithms are often the result of interactions
with physical entities~\cite{BRA20} (e.g., recommendation systems, clinical trials, advertising, etc.),
raising concerns about the privacy of participating entities.
For example, responses of an individual to medical treatments
can inadvertently reveal privacy-sensitive health information. Therefore, it
is essential to design learning algorithms that preserve the privacy of
reward sequences.

Furthermore, observations collected from multiple users or external resources are prone to
failures or corruptions. These corruptions are modeled by \emph{adversaries}, which can tamper with a fraction
of the observed rewards. Adversarial corruptions can be strategic,
(e.g., simultaneously hijacking the devices of multiple users), or random (such
as misclicks in the context of an ad campaign). Regardless of their nature,
they highlight the need for developing \emph{robust} learning algorithms that succeed in
the presence of such corruptions. Developing robust private policies has drawn considerable
attention in the past couple of years (\cite{NME22,liu2021robust,kothari2022private,ghazi2021robust,dimitrakakis2014robust,li2022robustness}). However, despite the importance of the bandits model, we are not aware of any provably robust and private policy for this model.

Lastly, in practice,
it is often desirable or even necessary for the learner to perform actions
in parallel. For example, ad campaigns present an assortment of advertisements
to multiple users at the same time and are only periodically recalibrated~\cite{BM07}.
Consequently, batch policies must optimally balance between parallelization,
which can offer significant time savings, and information exchange, which must
happen frequently enough to allow for exploration of the action space~\cite{EKMM21}.

In this paper, we develop a learning policy that addresses both privacy and robustness challenges, while enjoying the benefits of parallelization. Specifically, our policy protects the privacy of reward sequences by respecting the standard differential privacy measure, while withstanding an adversary that changes a constant fraction of the observed rewards in each batch. In the
remainder of this section, we formally introduce the problem and survey related
work in the bandit literature.

\subsection{Problem formulation and provable guarantees}
We study the \textbf{stochastic linear bandit} problem: given an action space
$\mathcal{A} \subset \mathbb{R}^d$ with $K$ elements satisfying $\max_{a \in \cA}
\norm{a}_2 \leq 1$, a learner ``plays'' actions $a \in \cA$ and receives rewards
\begin{equation}
    r_{a} := \ip{a, \theta^{\star}} + \eta, \quad
    \eta \sim \mathrm{SubG}(1),
\end{equation}
where $\theta^{\star}$ is an unknown vector in $\Rbb^d$ and $\mathrm{SubG}(1)$
denotes a zero-mean subgaussian random variable. Our assumption that $\abs{\cA} \leq K$
is without loss of generality, since our results extend to the infinite case by
a standard covering argument~\citep[Chapter 20]{LS20}. For simplicity, we also assume
that $\norm{\theta^{\star}} \leq 1$. Given a budget of
$T$ total actions, the goal of the learner is to minimize her \emph{expected regret}:
\begin{equation}
  \expec{R_T} :=
  \max_{a \in \cA} \sum_{t = 1}^T \ip{a - a_t, \theta^{\star}}
  \label{eq:expected-regret}
\end{equation}

\paragraph{Batched observations.}
In bandits problems with batch policies, the learner commits to a sequence (i.e., a
\textit{batch}) of actions and observes the rewards of the actions \textit{only
after the entire batch of actions has been played}. The learner may play multiple
batches of actions, whose sizes may be chosen adaptively, subject to the requirement
that the total number of batches does not exceed $B$ (in addition to the total number
of actions played not exceeding the budget $T$). We assume that $B$ is also known to
the learner.

\paragraph{Robustness.} We require that our algorithm is robust under
possibly adversarial corruptions suffered. In particular, we assume that an
adversary replaces every observation by arbitrary numbers with some small
probability $\alpha$. Thus, during each batch, the \emph{observed rewards} satisfy
\begin{equation}
  r_i = \begin{cases}
    \ip{a_i, \theta^{\star}} + \eta_i, & \text{with probability $1 - \alpha$}, \\
    \ast, & \text{with probability $\alpha$},
  \end{cases}, \quad i = 1, \dots, n,
  \label{eq:noise-model}
\end{equation}
where $\alpha \in [0, 1/4)$ is the failure probability and $\ast$ is an
arbitrary value.

\paragraph{Diffential privacy.}
Our other requirement is that the algorithm is
\emph{differentially private} (DP).
\begin{definition}[Differential Privacy for Bandits~\cite{BDT19}] \label{def:bandit-dp}
  A randomized mechanism $\cM$ for stochastic linear bandits is called
  $(\epspriv, \delpriv)$-differentially private if, for any two neighboring sequences of
  rewards $\cR = (r_1, \dots, r_T)$ and $\cR' = (r'_1, \dots, r'_T)$ where
  $r_i \neq r'_{i}$ for at most one index $i$, and any subset of outputs
  $O \in \cM^{T}$, it satisfies
  \begin{equation}
    \prob{\cM(\cR) \in O} \leq
      e^{\epspriv} \cdot \prob{\cM(\cR') \in O} + \delpriv.
  \end{equation}
\end{definition}
The main contribution of our paper is a batched arm elimination algorithm that satisfies
both desiderata, presented in detail in~\cref{sec:algorithm-and-main-results}. We assume
a \emph{distributed} setting where a central server takes on the role of the learner,
connected with several clients that report back rewards. The clients do not trust the
central server and therefore choose to privatize their reward sequences; this model is
better known as \textit{local differential privacy} (LDP)~\cite{KLN+11}.
Our algorithm addresses the following client response models:

% TODO: Describe these models in more detail.

\begin{enumerate}[label={\textbf{(M\arabic*)}},leftmargin=3em]
  \item Each reward $\bar{r}_i$ may be solicited from a different client $i$.
    \label{enum:privacy-local-1}
%     client returns at most $1$ response. \HE{I think what we really need is: Each reward may be solicited from a different client and hence w.l.g. we assume that every reward $\bar{r}_i$ is solicited from a different client $i$, and every client returns at most $1$ response.} \label{enum:privacy-local-1}
  \item Each client ``owns'' an action $a \in \cA$ and may report multiple rewards
    in each batch. \label{enum:privacy-local-2}
\end{enumerate}

\begin{remark}
  In Model~\ref{enum:privacy-local-1}, we may assume without loss of generality that
  every reward $\bar{r}_i$ is solicited from a different client, and thus every client
  returns at most $1$ response.
\end{remark}

Below, we provide informal statements for the expected regret that our algorithms
achieve under each model. While the regret under~\ref{enum:privacy-local-1} has
better dependence on the dimension $d$,~\ref{enum:privacy-local-2} leads to a
better dependence on the privacy parameter $\epspriv$. The latter
should not come as a complete surprise, since model~\ref{enum:privacy-local-2}
can be viewed as interpolating between
the \emph{local} and \emph{central} models of differential privacy. For simplicity,
we focus on the case where $B$ scales logarithmically in $T$, although our analysis
can be easily modified for general $B$.

\begin{theorem}[Informal]
  \label{theorem:M1-informal}
  Under Model~\ref{enum:privacy-local-1}, there is an $\epspriv$-locally differentially
  private algorithm that is robust to adversarial corruptions with expected regret satisfying
  \begin{align*}
    \expec{R_T} = \tilde{O}\left(
      \left[ \sqrt{dT} + T \max\set{\sqrt{\alpha d}, \alpha d} \right]
        \left(1 + \frac{1}{\epspriv}\right)
    \right)
\end{align*}
\end{theorem}
It is worth noting that in the non-private setting, the regret bound above scales as
$\tilde{O}(T\sqrt{\alpha d} + \sqrt{dT})$ when $\alpha < d$ and
$\tilde{O}(T \alpha d + \sqrt{dT})$ when $\alpha \geq d$.
Note that the total amount of corruption injected by the adversary is upper bounded by $C = \alpha T$.
Interestingly,
our result shaves off a factor of at least $\sqrt{d}$ compared to the regret bound of the best previous work on robust stochastic linear bandits~\cite{BLKS21},
which scales as $\tilde{O}(\sqrt{dT} + Cd^{3/2})$.

\begin{theorem}[Informal]
  \label{theorem:M2-informal}
  Under Model~\ref{enum:privacy-local-2}, there is an $\epspriv$-differentially private
  algorithm that is robust to adversarial corruptions with expected regret satisfying
  \[
    \expec{R_T} =
    \tilde{O}\left(d\sqrt{T} + \frac{d}{\epspriv}\right) + \tilde{O}\left(d^{3/2} \sqrt{\alpha} \left(
        T + d \sqrt{T} + \frac{d}{\epspriv}
      \right)\right) + \alpha T.
  \]
\end{theorem}
Compared to~\cref{theorem:M1-informal},
\Cref{theorem:M2-informal} yields an improved dependence on $T$ in the ``private''
part of the regret at the expense of an additional $\sqrt{d}$ factor in the
non-private part.

% TODO: Table below must be updated.

\subsection{Related work}
% \begin{table*}[h]
%     \centering
%     \begingroup
%     \renewcommand*{\arraystretch}{1.2}
%     \begin{tabular}{l c l c l} \\ \hline \toprule
%       \textbf{Reference} & \textbf{Private} & \textbf{Type of privacy} & \textbf{Robust} & \textbf{Regret bound} \\ \midrule
%       \cite{ZCH+20} & \cmark & Local  & \xmark   & $\tilde{O}(T^{3/4} / \epspriv)$ \\
%       \cite{HLWZ21} & \cmark & Local  & \xmark   & $\tilde{O}(d \sqrt{T} / \epspriv)$ \\
%       \cite{HGFD22} & \cmark & Local  & \xmark    & $\tilde{O}(d \sqrt{T} / \epspriv)$ \\
%       \cite{SS18}   & \cmark & Central & \xmark   & $\tilde{O}(\sqrt{dT} / \epspriv)$ \\
%       \citep[Theorem 1]{HGFD22} & \cmark & Central & \xmark    & $\tilde{O}(\sqrt{dT} + \frac{d}{\epspriv})$ \\
%       \citep[Theorem 2]{HGFD22} & \cmark & Local & \xmark      & $\tilde{O}\left(d\sqrt{T}\left(1 + \sfrac{1}{\epspriv}\right)\right)$ \\
%       \cite{BLKS21} & \xmark & -- & \cmark    & $\tilde{O}(\sqrt{dT} + C d^{3/2})$ \\
%       Our algorithm & \cmark & Local (\cref{theorem:M1-informal}) & \cmark &
%       $\tilde{O}\left((\sqrt{dT} + T \max\{\sqrt{\alpha d}, \alpha d\})\left(1 + \frac{1}{\epspriv}\right)\right)$ \\
%       Our algorithm & \cmark & Local (\cref{theorem:M2-informal}) & \cmark &
%         $\tilde{O}(???)$ \\
%         \bottomrule
%     \end{tabular}
%     \endgroup
%     \caption{Summary of regret bounds. Here, $\alpha$ denotes the fraction of corruption. \HE{Needs to be updated}
%     \label{table:literature}
% \end{table*}
In this section, we survey related work in the bandit literature that
addresses differential privacy and/or robustness to corruptions.
% Table~\ref{table:literature} offers a quick comparison with the literature.
We note that, to the best of our knowledge,
our work is the first to simultaneously provide robustness and differential privacy
guarantees.

\paragraph{Differential privacy in linear bandits.}
Differential privacy has been well-studed in the context of bandit learning.
In the central DP model, which is the focus of this paper, \cite{SS18} proved a
lower bound of $\Omega(\sqrt{T} + \frac{\log(T)}{\epspriv})$ on the expected regret
and proposed a private variant of the \texttt{LinUCB} algorithm with additive noise
that achieves expected regret of $\tilde{O}(\sqrt{T} + \sqrt{T} /\epspriv)$.
In recent work,~\cite{HGFD22} proposed a private variant of the arm elimination
algorithm that obtains a regret bound of $O(\sqrt{T \log T} + \frac{\log^2(T)}{\epspriv})$
which is tight up to logarithmic factors. While conceptually similar to that of~\cite{HGFD22},
our algorithm guarantees differential privacy and robustness to corrupted
observations simultaneously and maintains an order-optimal regret bound.

In the local DP model,~\cite{ZCH+20} used a reduction to private bandit convex
optimization to achieve expected regret $\tilde{O}(T^{3/4} / \epspriv)$.
Under additional distributional assumptions on the action set, this was
improved to $\tilde{O}(T^{1/2} / \epspriv)$ by~\cite{HLWZ21}. The same rate was obtained by~\cite{HGFD22},
who removed the requirement that actions are generated from a distribution.
Finally, a recent line of work focused on so-called
\emph{shuffle differential privacy}~\cite{BEM+17,Cheu21}, wherein
a trusted \emph{shuffler} can preprocess client responses before
transmitting them to the central server. A sequence of works~\cite{TKMS21,CZ22a,GCPP22,HGFD22}
proposed shuffle-DP algorithms for linear bandits,
with~\cite{HGFD22} achieving essentially
the same regret bound as in the central DP setting.

\paragraph{Robustness to adversarial attacks.}
Recent work proposed various adversarial attacks in the bandit setting, as well
as algorithms to protect against them. \cite{LMP18} (and~\cite{GKT19} in a follow-up work)
study multi-armed bandits with
\emph{adversarial scaling}, wherein an adversary can shrink the means of the arm
distributions in each round, and propose robust algorithms for this setting. The
corruption in this work differs from our setting, where the adversary can replace
a random fraction of rewards arbitrarily.
The works of~\cite{JLMZ18,LS19,GRM+20} study multi-armed and contextual
bandit algorithms from the attacker perspective, demonstrating how an adversary
can induce linear regret with logarithmic effort.

\cite{LLS19} and~\cite{BLKS21} study additive adversarial
corruptions in contextual bandits. In particular, they assume that the observed
reward in round $i$ suffers an additive perturbation by $c_{i}(a_i)$, where $a_i$
is the $i^{\text{th}}$ context and $c_{i}: \cA \to [-1, 1]$ is a context-dependent
corruption function. Crucially, the adversary is subject to a budget constraint
given some budget $C$ unknown to the learner:

\begin{equation}
  \sum_{i=1}^T \max_{a \in \cA}\abs{c_i(a)} \leq C.
  \label{eq:attacker-budget}
\end{equation}

In~\cite{LLS19}, the authors present a robust exploration algorithm for contextual
bandits using the L{\"o}ewner ellipsoid. Letting $\Delta$ denote the gap between
the highest and
lowest expected rewards, their algorithm achieves a regret of
${O}\big(\frac{d^{5/2} C \log T}{\Delta} + \frac{d^6 \log^2 T}{\Delta^2}\big)$,
under the key assumption that the action space $\cA$ is a full-dimensional polytope,
and requires no knowledge of the corruption budget $C$.

On the other hand, the work of~\cite{BLKS21} introduces a robust variant of the
phased arm elimination algorithm for stochastic linear bandits that achieves an
expected regret of $\tilde{O}(\sqrt{dT} + Cd^{3/2})$, assuming the budget $C$ is
known to the learner; for unknown budgets, an additional $C^2$ factor appears in
the regret bound. Our work deviates from that of~\cite{BLKS21} in the sense that
we measure corruption using the probability $\alpha$ of an adversary interfering with each
observation; moreover, assuming that $C$ scales as $\alpha T$, our work shaves off
a $\sqrt{d}$ factor from the result of~\cite{BLKS21} in certain regimes, while it
also ensures differential privacy.

\subsection{Notation}
We let $\ip{x, y} := x^{\T} y$ denote the Euclidean inner product with induced
norm $\norm{x} = \sqrt{\ip{x, x}}$ and write $\Sbb^{d-1} := \{x \in \Rbb^d
\mid \norm{x}_2 = 1\}$ for the unit sphere in $d$ dimensions. When $M$ is a
positive-definite matrix, we write $\norm{x}_{M} := \sqrt{\ip{x, Mx}}$ for
the norm induced by $M$. Finally, we write $\opnorm{A} := \sup_{x \in \Sbb^{n-1}}
\norm{Ax}_2$ for the $\ell_2 \to \ell_2$ operator norm of a matrix $A \in \Rbb^{m \times n}$.

\subsection{Coresets and G-optimal designs}
\label{sec:coresets}
Our algorithms make use of \textit{coresets}, which in turn are formed with
the help of a concept called \textit{G-optimal design}. We formally define
this concept below.
\begin{definition}[G-optimal design]
  \label{def:g-optimal-design}
  Let $\cA \subset \Rbb^d$ be a finite set of vectors and let $\pi: \cA \to [0, 1]$
  be a probability distribution on $\cA$ satisfying $\sum_{a \in \cA} \pi(a) = 1$.
  Then $\pi$ is called a \textit{G-optimal design} for $\cA$ if it satisfies
  \begin{equation}
    \pi = \argmin_{\pi'} \set{ \max_{a \in \cA} \norm{a}^2_{M^{-1}(\pi)} },
    \label{eq:g-opt-design-prob}
  \end{equation}
  where $M(\pi) := \sum_{a \in \cA} \pi(a) aa^{\T}$.
\end{definition}

A standard result in experiment design~\citep[Theorem 21.1]{LS20} shows that
the optimal value of~\eqref{eq:g-opt-design-prob} is equal to $d$. Moreover,
it is possible to find a probability distribution $\pi$ satisfying the following:
\begin{definition}[Approximate G-optimal design]
  \label{def:approx-g-opt-design}
  Let $\cA \subset \Rbb^d$ be a finite set of vectors and let $\pi: \cA \to [0, 1]$
  be a probability distribution on $\cA$.
  We call $\pi$ an \textit{approximate G-optimal design} for $\cA$ if it satisfies
  \begin{equation}
    \max_{a \in \cA} \norm{a}^2_{M^{-1}(\pi)} \leq 2d, \quad
    \abs{\supp(\pi)} \leq C d \log \log d
    \label{eq:approx-g-opt-design-prob}
  \end{equation}
  where $M(\pi) := \sum_{a \in \cA} \pi(a) aa^{\T}$ and $C$ is
  a universal constant. Moreover, $\pi$ can be found in time
  $O(d \log \log d)$.
\end{definition}

Given a (approximate) G-optimal design in the sense of~\cref{def:g-optimal-design}
or~\cref{def:approx-g-opt-design}, a \textit{coreset} $\cS_{\cA}$ of total size $n$
is a multiset $\set{a_1, \dots, a_n}$ where each action $a \in \supp(\pi)$ appears
a total of $n_{a} := \ceil{\pi(a) \cdot n}$ times.

\section{Algorithm and main results} \label{sec:algorithm-and-main-results}
To minimize the regret of the learner, we use a variation of the standard arm
elimination algorithm~\cite{LS20}. In this algorithm, the learner
uses batches of actions to construct confidence intervals for the optimal rewards
and eliminates a set of suboptimal arms in each round based on their performance
on the current batch.
While the vanilla arm elimination algorithm \emph{is neither robust nor
differentially private}, we develop a variant that simultaneously ensures
both these properties. An additional attractive property of our algorithm is
that its implementation only requires a simple modification.

\subsection{Our approach}
To motivate our approach, we first sketch a naive attempt at modifying the
arm elimination algorithm and briefly explain why it is unable to achieve good
regret guarantees.

Recall that in the standard arm elimination algorithm, the learner
first forms a so-called \emph{coreset} of the
action space $\cA$, which is a multiset of vectors $a_1, \dots, a_n \in \cA$, and plays all the
actions $a_j$ receiving rewards $r_{j}$.
To prune the action space, the learner first computes the least-squares
estimate:
\begin{equation}
  \widehat{\theta} := \left( \sum_{j = 1}^n a_j a_j^{\T} \right)^{-1}
  \sum_{j = 1}^n a_j r_j,
  \label{eq:vanilla-least-squares-estimate}
\end{equation}
and chooses a suitable threshold $\gamma$ to eliminate arms with
\[
  \bigip{a, \widehat{\theta}} < \max_{j = 1, \dots, n} \bigip{a_j, \widehat{\theta}} -
  2 \gamma.
\]

Clearly, the arm elimination algorithm interacts with the rewards directly only
when forming the least squares estimate $\widehat{\theta}$. Therefore, estimating
$\widehat{\theta}$ with a differentially private algorithm is sufficient to protect
the privacy of rewards. Likewise, computing $\widehat{\theta}$ robustly will ensure
robustness of the overall algorithm.

The main idea behind our arm elimination variant is the following. First, let
us dispense with the differential privacy requirement. Notice that in the absence
of corruptions, $\widehat{\theta}$ is the empirical mean of the sequence of
variables $\set{Z_1, \dots, Z_n}$:
\begin{equation*}
    Z_j := \left( \sum_{i = 1}^n a_i a_i^{\T} \right)^{-1} r_j a_j.
\end{equation*}
To compute $\widehat{\theta}$ robustly, one may attempt to run an algorithm
such as the geometric median. However, the approximation guarantee of the
geometric median method scales proportionally to $\max_{j \in [n]}
\| Z_j - \widehat{\theta} \|$, for which worst-case bounds are overly
pessimistic. Indeed, letting $M$ denote the Gram matrix of the coreset used
in the current arm elimination round, a tedious but straightforward calculation
shows that these bounds scale as $\kappa_2(M)$, the condition number of $M$.
In turn, the latter quantity depends on the geometry of the maintained action
set and is difficult to control in general. For example, even if the original
action set is ``well-conditioned'', that property will not necessarily hold
throughout the algorithm.

To work around this issue, we take advantage of the probabilistic nature of
the adversary. The main idea is that, \textit{in expectation}, the least-squares
estimate computed over the subset of ``clean'' rewards $\Igood$, given by
\begin{equation}
  \widehat{\theta}_{\Igood} = \left(\sum_{i=1}^n a_i a_i^{\T}\right)^{-1}
  \sum_{j \in \Igood} a_j r_j,
  \label{eq:lsq-est-good}
\end{equation}
is close to the true least-squares estimate in the absence of any corruptions.
While the set $\Igood$ is not known a-priori to the learner, we may still
estimate $\widehat{\theta}_{\Igood}$ from~\cref{eq:lsq-est-good} using a
well-known spectral filtering algorithm from the robust statistics literature.
In doing so, we reduce the problem of robust linear regression (with a
\textit{fixed} design matrix) to that of robust mean estimation (over an
appropriately weighted set of inputs). We mention in passing that the recent work
of~\cite{CMY22} also develop a distribution-free algorithm for robust linear
regression which applies to a more general class of problems. However, their
algorithm requires repeatedly solving a semidefinite program, while our
spectral filtering-based method is simpler to implement.

In what follows, we describe our robust linear regression primitive and
state its theoretical approximation guarantees, and finally sketch how to take
advantage of it to design a robust and diffentially private algorithm for batched bandits.

\begin{algorithm}[!ht]
    \caption{Robust arm elimination}
    \begin{algorithmic}[1]
        \STATE \textbf{Input}: action space $\cA$, $T$, $B$,
        failure prob. $\delta$, corruption prob. $\alpha \in (0, 1/4)$, truncation parameter $\nu > 0$.
        \STATE Set $\cA_0 := \cA$, $q = T^{1/B}$
        \FOR{$i = 1, \dots, B - 1$}
            \STATE Compute approximate $G$-optimal design $\pi$ with
            $\abs{\supp(\pi)} \lesssim d \log \log d$.
            \STATE Form a coreset $\cS_{\cA_{i-1}}$ by playing each distinct
            $a \in \supp(\pi)$ a total of
            \[
              n_{a} = \begin{cases}
                \ceil{q^i \pi(a)}, & \text{under Model~\ref{item:M1}}; \\
                \ceil{q^i \max\set{\pi(a), \nu}}, & \text{under Model~\ref{item:M2}}.
              \end{cases}.
            \]
            \STATE Play actions $a_j \in \cS_{\cA_{i-1}}$ and collect rewards $r_j$
            according to~\eqref{eq:noise-model}.
            \STATE Letting $M_n := \sum_{i = 1}^n a_i a_i^{\T}$, compute
            \[
              \widetilde{w}_i := \texttt{Filter}\left(\set{M_n^{-1/2} a_i r_i}_{i=1}^n, \frac{\max_{a \in \cA} \norm{a}_{M_n^{-1}}^2 \sum_{i = 1}^n r_i^2}{n}\right),
              \quad \text{using Alg.~\ref{alg:spectral-filtering}}.
            \]
            \STATE Compute $\widetilde{\theta}_i := n M_n^{-1/2} \widetilde{w}$.
            \STATE Set the elimination threshold
            \[
              \resizebox{.9\hsize}{!}{$
              \gamma_i := \begin{cases}
                  \sqrt{d}\left(\sqrt{\log(q^i / \delta)} + \frac{\log(q^i / \delta)}{\epspriv}\right)
                  (\sqrt{\alpha} + \alpha \sqrt{d}) +
                  \alpha + \sqrt{\frac{d \log(1 / \delta)}{q^i}}\left(
                    1 + \frac{\sqrt{\log(1 / \delta)}}{\epspriv}
                  \right), & \text{\ref{item:M1}}; \\[2ex]
                    \sqrt{\frac{d \log(1 / \delta)}{\nu m}} \left(1 + \frac{1}{\epspriv}\sqrt{\frac{\log(1/\delta)}{\nu m}}\right) + 2d \left(1 + \sqrt{\frac{\log (k / \delta)}{\nu m}} +
                                 \frac{\log(k / \delta)}{\nu m \epspriv} \right)
                            \left(\sqrt{k \alpha} + \sqrt{\alpha \log(1 / \delta)}\right) + \alpha,
              & \text{\ref{item:M2}},
              \end{cases}
              $}
            \]
            where $k := \abs{\mathrm{supp}(\pi)}$ in the second option.
            \STATE Eliminate suboptimal arms:
            \[
                \cA_{i} := \set{
                    a \in \cS_{\cA_{i-1}} \mid
                    \langle a, \widetilde{\theta}_i\rangle \geq
                    \max_{a' \in \cS_{\cA_{i-1}}} \langle a', \widetilde{\theta}_i \rangle
                    - 2 \gamma_i.
                },
            \]
        \ENDFOR
        \STATE Play the ``best'' action in $\cS_{\cA_{B-1}}$ in the last round.
    \end{algorithmic}
    \label{alg:arm-elimination-robust}
\end{algorithm}

\subsection{Robust linear regression with fixed designs} \label{sec:robust-linear-regression}
In this section, we describe an efficient algorithm for Huber-robust linear
regression with a fixed design matrix. In particular, we let the
(clean) set of observations satisfy
\begin{equation}
  y_i = \ip{a_i, \theta^{\star}} + \eta_i, \quad i = 1, \dots, n.
  \label{eq:linear-regression-setup}
\end{equation}
where $\eta_i$ are independent noise realizations and $a_i$ are
design vectors. The least-squares estimate of $\theta^{\star}$ is given by
\[
  \widehat{\theta} := M_n^{-1} \sum_{i=1}^n y_i a_i, \; \quad
  M_n := \sum_{i = 1}^n a_i a_i^{\T}.
\]
Now, suppose that an adversary corrupts each $y_i$ independently with probability
$\alpha \in (0, 1/4)$, so the learner observes
\begin{equation}
  \hat{y}_i = \begin{cases}
    y_i, & \text{if $Z_i = 1$}, \\
    \ast, & \text{otherwise}
  \end{cases}, \quad
  Z_i \sim \mathrm{Ber}(1 - \alpha).
  \label{eq:robust-linreg-observed-rewards}
\end{equation}
The goal is to estimate the least-squares solution $\widehat{\theta}$
robustly. Our strategy will be to first estimate the least-squares
solution over the subset of ``good'' indices $G_0$:
\begin{equation}
  \theta_{G_0} = \sum_{i \in G_0} M_n^{-1} a_i y_i, \quad
  G_0 = \set{i \mid Z_i = 1}.
  \label{eq:robust-linreg-good-mean}
\end{equation}
To estimate $\theta_{G_0}$, we will apply the well-known (randomized)
spectral filtering algorithm for robust mean estimation (see, e.g.,~\cite{DK19,PBR19}),
provided in~\cref{alg:spectral-filtering} for completeness,
to the components of the least-squares solution after an appropriate
reweighting. In particular, we will estimate
\begin{equation}
  \begin{aligned}[t]
    \gamma_{\set{a_i}_{i=1}^n}
                  & := \frac{\max_{a \in \cA} \norm{a}_{M_n^{-1}}^2 \sum_{i=1}^n y_i^2}{n}; \\
    \widetilde{w} &:= \texttt{Filter}\left(
    \set{M_n^{-1/2} a_i y_i}_{i=1}^n, \gamma_{\set{a_i}_{i=1}^n}\right); \\
    \widetilde{\theta} &:= n M_n^{-1/2} \widetilde{w}
  \end{aligned}
  \label{eq:robust-linreg-estimate}
\end{equation}
We prove the following guarantee for this method, whose is deferred to
\cref{sec:appendix:robust-linear-regression-proofs}.
\begin{proposition}
  \label{prop:robust-confidence-intervals}
  Fix a $\delta \in (0, 1)$, $a \in \cA$ and let $e_i = y_i - \ip{a_i, \theta^{\star}}$.
  Then with probability at least $1 - 2\delta$, we have
  \begin{equation}
    \big|\langle a, \widetilde{\theta} - \theta^{\star} \rangle\big| \lesssim
    \begin{aligned}[t]
      & \max_{a \in \cA} \norm{a}_{M_n^{-1}}^2
      \sqrt{n \sum_{i = 1}^n y_i^2}
      \left(\alpha + \frac{\log(1 / \delta)}{n}\right)^{1/2} \\
      & \quad + \max_{a \in \cA} \norm{a}_{M_n^{-1}}^2
      \sqrt{\sum_{i = 1}^n y_i^2}
      + \sqrt{\alpha \log(1 / \delta)} \\
      & \quad +
      \sum_{i = 1}^n e_i \ip{a, M_n^{-1} a_i} + \alpha.
    \end{aligned}
  \label{eq:robust-confidence-intervals}
  \end{equation}
\end{proposition}
We use this proposition in the next section to design robust and differentially private algorithms for stochastic linear bandits.

\begin{algorithm}[t]
  \caption{\texttt{Filter}($S := \set{X_{i}}_{i=1}^m$, $\lambda$)}
  \begin{algorithmic}[1]
    \STATE Compute empirical mean and covariance:
    \[
      \theta_S := \frac{1}{\abs{S}} \sum_{i \in S} X_i, \;
      \Sigma_{S} := \frac{1}{\abs{S}} \sum_{i \in S} (X_i - \theta_S)(X_i - \theta_S)^{\T}.
    \]
    \STATE Compute leading eigenpair $(\mu, v)$ of $\Sigma_{S}$.
    \IF{$\mu < 4\lambda$}
      \STATE \textbf{return} $\theta_{S}$
    \ELSE
      \STATE Compute outlier scores $\tau_i := \ip{v, X_i - \theta_S}^2$ for all $i$.
      \STATE Sample an element $Y$ with $\prob{Y = X_i} \propto \tau_i$
      \STATE \textbf{return} $\texttt{Filter}(S \setD \set{Y}, \lambda)$
    \ENDIF
  \end{algorithmic}
  \label{alg:spectral-filtering}
\end{algorithm}

\section{Robust differentially private bandits} \label{sec:dp-analysis}
In this section, we consider the requirement of \emph{differential
privacy}. In particular, we assume that the learner is an \emph{untrusted}
server; every client must therefore privatize their rewards before reporting
them to the learner. Recall that we consider two different models for
generating client responses:

\begin{enumerate}[leftmargin=4em, label=(\textbf{M\arabic*}), ref=(M\arabic*)]
  \item Every reward is obtained from a distinct client.
    \label{item:M1}
  \item All rewards associated with a distinct action $a$ are obtained
    from the same client. \label{item:M2}
\end{enumerate}

\Cref{alg:arm-elimination-robust} documents the parameter choices under each
of the models above. For our regret analysis, we rely on the following facts
for each round $i$.

\paragraph{Fact 1: The optimal arm is not eliminated.}
Let $a^{\star}$ denote the ``optimal'' action in the sence of maximizing
the inner products $\ip{a, \theta}$. Then, with high probability,
\begin{align*}
  \langle a, \widetilde{\theta} \rangle -
  \langle a^{\star}, \widetilde{\theta} \rangle &=
  \langle a, \theta^{\star}\rangle + \langle a, \widetilde{\theta} - \theta^{\star} \rangle - \langle a^{\star}, \theta^{\star} \rangle
  - \langle a^{\star}, \widetilde{\theta} - \theta^{\star}\rangle \\
  &\leq
  \langle a - a^{\star}, \theta^{\star} \rangle + 2 \gamma_i \\
  &\leq 2 \gamma_i,
\end{align*}
using the bound on the difference in the penultimate inequality and the fact that
$\langle a, \theta^{\star} \rangle \leq \langle a^{\star}, \theta^{\star} \rangle$
in the last inequality. Thus, $a^{\star}$ always satisfies the condition of the algorithm and is not eliminated.

\paragraph{Fact 2: Surviving arms have bounded gap.}
Fix an arm $a$ and let $\Delta := \langle a^{\star} - a, \theta^{\star} \rangle$ be its gap. We have
\[
\begin{aligned}
\langle a^{\star} - a, \widetilde{\theta} \rangle 
&\geq \langle a^{\star}, \theta^{\star} \rangle - \gamma_i
- (\langle a, \theta^{\star}\rangle + \gamma_i) \\
&\geq \Delta - 2 \gamma_i.
\end{aligned}
\]
Now, let $i$ be the smallest positive integer such that $\gamma_i < \Delta / 4$.
Then the above implies that
$$\langle a^{\star} - a, \widetilde{\theta} \rangle \geq 2 \gamma_i.$$ 
Consequently, any arm $a$ with gap $\Delta_a > 4 \gamma_i$ for some index $i$ will
be eliminated at the end of that round. Therefore, all arms that are active at the beginning of round $i$ will necessarily satisfy $\Delta_a \leq 4 \gamma_{i  - 1}$.

\subsection{Local differential privacy under~\ref{item:M1}} \label{sec:M1-DP}
In this setting, we can achieve pure LDP using the Laplace mechanism~\cite{DR14}.
We define
\[
  \cM(r) = r + \xi, \quad \xi \sim \mathsf{Lap}\left(\frac{2}{\epspriv}\right),
\]
where $\epspriv$ is a desired privacy parameter. Then, when queried for a response,
client $i$ reports the privatized reward:
\begin{equation}
  \hat{r}_i = \cM(r_i) = \ip{a_i, \theta^{\star}} + \eta_i + \xi_i, \quad
  \xi_i \sim \mathsf{Lap}\left(\frac{2}{\epspriv}\right).
\end{equation}
For our regret analysis, we must control the
effect of the noise $\xi_i$ on the confidence interval~\eqref{eq:robust-confidence-intervals}.
\begin{lemma}
  \label{lemma:M1-error-contrib}
  Under the model~\ref{item:M1}, with probability at least $1 - 2\delta$ we have
  \begin{equation}
    \sum_{i = 1}^n e_i \ip{a, M^{-1} a_i} \leq
    \norm{a}_{M^{-1}} \sqrt{\log(1 / \delta)} \left(
      c_1 + \frac{c_2 \sqrt{\log(1 / \delta)}}{\epspriv}
    \right).
  \end{equation}
\end{lemma}
We now bound the term $\norm{a}_{M^{-1}}$ for an arbitrary round of
the arm elimination algorithm. Below, we let $n$ denote the
total number of actions played during the current round.
\begin{lemma} \label{lemma:norm-bound-g-optimal-M1}
  Under Model~\ref{item:M1}, we have the bound:
  \begin{equation}
    \max_{a \in \cA} \norm{a}_{M^{-1}}^2 \leq \frac{2d}{n},
  \end{equation}
  where $n$ is the total number of actions played in the current
  round.
\end{lemma}

In addition, we control the contribution of $\sqrt{\sum_{i=1}^n y_i^2}$ to the
robust confidence interval in~\cref{prop:robust-confidence-intervals}.
\begin{lemma} \label{lemma:M1-y-norm}
  With probability at least $1 - \delta$, we have
  \begin{equation}
    \sqrt{\sum_{i = 1}^n y_i^2} \lesssim
    \sqrt{n} \left(
      1 + \sqrt{\log(n / \delta)} + 
      \frac{\log(n / \delta)}{\epspriv}
    \right).
  \end{equation}
\end{lemma}

\begin{theorem}
  \label{theorem:M1-regret}
  Under Model~\ref{item:M1}, the expected regret of~\cref{alg:arm-elimination-robust}
  is equal to
  \begin{equation}
    \sqrt{Td \log(T / \delta)}\left(1 + \frac{\sqrt{\log(T/\delta)}}{\epspriv}\right) + T \sqrt{\log(T / \delta)}\max\set{\sqrt{\alpha d}, \alpha d}\left(
      1 + \frac{\log(T/\delta)}{\epspriv}
    \right),
    \label{eq:M1-regret}
  \end{equation}
  up to a dimension-independent multiplicative constant.
\end{theorem}
\subsection{Local differential privacy under~\ref{item:M2}} \label{sec:M2-DP}
% TODO: Describe model
In this setting, every client achieves differential privacy by aggregating
their responses before transmitting them to the server.
In particular, let $n_{a}$ denote the number of times action $a$ is played
during the current round. The parameter $n_{a}$ can be considered public,
since it is known to the untrusted server. Then, client $a$ may report
\begin{equation}
  \hat{r}_a =
    \cM_{a}\left(\frac{1}{n_{a}} \sum_{i = 1}^{n_{a}} \ip{a, \theta^{\star}} + \eta_i\right) = \frac{1}{n_a} \sum_{i = 1}^{n_{a}} \ip{a, \theta^{\star}} + \eta_i + \xi_{a},
  \label{eq:ldp-averaged}
\end{equation}
where $\eta_i \sim \mathrm{SubG}(1)$ and $\xi_{a}$ is Laplace noise.
The amount of noise needed to achieve privacy scales inversely with $n_{a}$.
\begin{lemma} \label{lemma:laplace-noise-ldp-averaged}
  With $\xi_{a} \sim \mathsf{Lap}\left(\frac{2}{n_{a} \epspriv}\right)$, the mechanism
  $\cM_{a}$ in~\eqref{eq:ldp-averaged} is $\epspriv$-differentially private.
\end{lemma}
Recall that in this model, the arm elimination algorithm follows the 
modifications below:
\begin{enumerate}
  \item We receive $\abs{\supp(\pi)}$ distinct responses in each round.
  \item Every action $a \in \mathrm{supp}(\pi)$ is played a total of
    $n_{a} = \ceil{m \max\set{{\pi}(a), \nu}}$.
\end{enumerate}
Now, let $\cH = \set{a \mid a \in \mathrm{supp}({\pi})}$. We have the following analogue
of~\cref{lemma:M1-error-contrib}:
\begin{lemma}
  \label{lemma:M2-error-contrib}
  Under the model~\ref{item:M2}, with probability at least $1 - 2\delta$ we have
  \begin{equation}
    \sum_{v \in \cH} e_{v} \ip{a, M^{-1} v} \leq
    \frac{\norm{a}_{M^{-1}}}{\sqrt{\nu m}} \sqrt{\log(1 / \delta)} \left(
      c_1 + \frac{c_2}{\epspriv} \sqrt{\frac{\log(1 / \delta)}{\nu m}}
    \right),
  \end{equation}
  where $e_{v} = \cM(r_{v}) - \ip{v, \theta^{\star}}$ and $M =
  \sum_{a \in \cH} aa^{\T}$.
\end{lemma}

\begin{lemma} \label{lemma:norm-bound-g-optimal-M2}
  Under Model~\ref{item:M2}, we have the bound:
  \begin{equation}
    \max_{a \in \cA} \norm{a}_{M^{-1}}^2 \leq 2d.
  \end{equation}
\end{lemma}

We also have the following analogue of~\cref{lemma:M1-y-norm}.
\begin{lemma} \label{lemma:M2-y-norm}
  With probability at least $1 - \delta$, we have
  \begin{equation}
    \sqrt{\sum_{v \in \mathrm{supp}(\hat{\pi})} y_v^2} \lesssim
    1 + \sqrt{\frac{\log(\abs{\mathrm{supp}({\pi})} / \delta)}{\nu m}} +
      \frac{\log(\abs{\mathrm{supp}({\pi})} / \delta)}{\nu m \epspriv}.
    \label{eq:M2-y-norm}
  \end{equation}
\end{lemma}

\begin{theorem}
  \label{theorem:M2-regret}
  Under Model~\ref{item:M2}, the regret of~\cref{alg:arm-elimination-robust} scales as
  \begin{equation}
    \begin{aligned}[t]
      & (1 + \nu d \log \log d) \left(
        \sqrt{\frac{d T \log(1 / \delta)}{\nu}} +
        \frac{\log(1 / \delta) \log(T) \sqrt{d}}{\epspriv \sqrt{\nu}}
      \right) \\
      & + 2d \left(\sqrt{\alpha d \log \log d} + \sqrt{\alpha \log(1 / \delta)}\right)
      \left(T + \sqrt{\frac{T d \log \log d / \delta)}{\nu}} + \frac{\log(d \log \log d / \delta)}{\nu} \frac{\log T}{\epspriv}\right) \\
      & + \alpha T.
    \end{aligned}
    \label{eq:M2-regret}
  \end{equation}
\end{theorem}

\section{Conclusion}
In this paper we presented a robust and $\epspriv$-LDP policy for
batched stochastic linear bandits with an expected regret $\expec{R_T} = \tilde{O}\left(
[\sqrt{dT} + T\max\{\sqrt{\alpha d}, \alpha d\}](1 + \sfrac{1}{\epspriv})\right)$, where
$\alpha$ is the probability of corruption of each reward, which only requires a logarithmic
number of batch queries. In the absence of corruption ($\alpha = 0$), our regret matches
that of the best-known non-robust differentially private algorithm~\cite{HGFD22}. On the
other hand, when no differential privacy is required, our regret bounds shaves off a
factor of $\sqrt{d}$ compares to previous work on robust linear bandits~\cite{BLKS21}.
In addition, a variant of our policy is immediately applicable to a differential privacy
model that interpolates between the local and central settings and achieves improved
dependence on the privacy parameter $\epspriv$.

While simple to implement, our algorithms require the learner to provide an upper bound
on the corruption probability $\alpha$, which may be difficult to estimate in practice.
We leave the task of designing an adaptive policy as exciting future work. At the same
time, it is unclear if our regret bounds for the privacy model~\ref{item:M2} are tight
(in terms of the dependence on $\epspriv$ and $d$). A natural question left open by our
work is constructing tight lower bounds in this setting.

\bibliographystyle{unsrtnat}
\bibliography{main}

\begin{thebibliography}{34}
\providecommand{\natexlab}[1]{#1}
\providecommand{\url}[1]{\texttt{#1}}
\expandafter\ifx\csname urlstyle\endcsname\relax
  \providecommand{\doi}[1]{doi: #1}\else
  \providecommand{\doi}{doi: \begingroup \urlstyle{rm}\Url}\fi

\bibitem[Bouneffouf et~al.(2020)Bouneffouf, Rish, and Aggarwal]{BRA20}
Djallel Bouneffouf, Irina Rish, and Charu Aggarwal.
\newblock Survey on applications of multi-armed and contextual bandits.
\newblock In \emph{2020 IEEE Congress on Evolutionary Computation (CEC)}, pages
  1--8. IEEE, 2020.

\bibitem[Esfandiari et~al.(2022)Esfandiari, Mirrokni, and Narayanan]{NME22}
Hossein Esfandiari, Vahab Mirrokni, and Shyam Narayanan.
\newblock Tight and robust private mean estimation with few users.
\newblock In \emph{Proceedings of the 39th International Conference on Machine
  Learning}, volume 162 of \emph{Proceedings of Machine Learning Research},
  pages 16383--16412. PMLR, 17--23 Jul 2022.

\bibitem[Liu et~al.(2021)Liu, Kong, Kakade, and Oh]{liu2021robust}
Xiyang Liu, Weihao Kong, Sham Kakade, and Sewoong Oh.
\newblock Robust and differentially private mean estimation.
\newblock \emph{Advances in Neural Information Processing Systems},
  34:\penalty0 3887--3901, 2021.

\bibitem[Kothari et~al.(2022)Kothari, Manurangsi, and
  Velingker]{kothari2022private}
Pravesh Kothari, Pasin Manurangsi, and Ameya Velingker.
\newblock Private robust estimation by stabilizing convex relaxations.
\newblock In \emph{Conference on Learning Theory}, pages 723--777. PMLR, 2022.

\bibitem[Ghazi et~al.(2021)Ghazi, Kumar, Manurangsi, and
  Nguyen]{ghazi2021robust}
Badih Ghazi, Ravi Kumar, Pasin Manurangsi, and Thao Nguyen.
\newblock Robust and private learning of halfspaces.
\newblock In \emph{International Conference on Artificial Intelligence and
  Statistics}, pages 1603--1611. PMLR, 2021.

\bibitem[Dimitrakakis et~al.(2014)Dimitrakakis, Nelson, Mitrokotsa, and
  Rubinstein]{dimitrakakis2014robust}
Christos Dimitrakakis, Blaine Nelson, Aikaterini Mitrokotsa, and Benjamin~IP
  Rubinstein.
\newblock Robust and private bayesian inference.
\newblock In \emph{International Conference on Algorithmic Learning Theory},
  pages 291--305. Springer, 2014.

\bibitem[Li et~al.(2022)Li, Berrett, and Yu]{li2022robustness}
Mengchu Li, Thomas~B Berrett, and Yi~Yu.
\newblock On robustness and local differential privacy.
\newblock \emph{arXiv preprint arXiv:2201.00751}, 2022.

\bibitem[Bertsimas and Mersereau(2007)]{BM07}
Dimitris Bertsimas and Adam~J. Mersereau.
\newblock A learning approach for interactive marketing to a customer segment.
\newblock \emph{Operations Research}, 55\penalty0 (6):\penalty0 1120--1135,
  2007.
\newblock \doi{10.1287/opre.1070.0427}.

\bibitem[Esfandiari et~al.(2021)Esfandiari, Karbasi, Mehrabian, and
  Mirrokni]{EKMM21}
Hossein Esfandiari, Amin Karbasi, Abbas Mehrabian, and Vahab Mirrokni.
\newblock Regret bounds for batched bandits.
\newblock In \emph{Proceedings of the AAAI Conference on Artificial
  Intelligence}, volume~35, pages 7340--7348, 2021.

\bibitem[Lattimore and Szepesv{\'a}ri(2020)]{LS20}
Tor Lattimore and Csaba Szepesv{\'a}ri.
\newblock \emph{Bandit algorithms}.
\newblock Cambridge University Press, 2020.

\bibitem[{Basu} et~al.(2019){Basu}, {Dimitrakakis}, and {Tossou}]{BDT19}
Debabrota {Basu}, Christos {Dimitrakakis}, and Aristide {Tossou}.
\newblock {Differential Privacy for Multi-armed Bandits: What Is It and What Is
  Its Cost?}
\newblock \emph{arXiv e-prints}, art. arXiv:1905.12298, May 2019.

\bibitem[Kasiviswanathan et~al.(2011)Kasiviswanathan, Lee, Nissim,
  Raskhodnikova, and Smith]{KLN+11}
Shiva~Prasad Kasiviswanathan, Homin~K Lee, Kobbi Nissim, Sofya Raskhodnikova,
  and Adam Smith.
\newblock What can we learn privately?
\newblock \emph{SIAM Journal on Computing}, 40\penalty0 (3):\penalty0 793--826,
  2011.

\bibitem[Bogunovic et~al.(2021)Bogunovic, Losalka, Krause, and
  Scarlett]{BLKS21}
Ilija Bogunovic, Arpan Losalka, Andreas Krause, and Jonathan Scarlett.
\newblock {Stochastic Linear Bandits Robust to Adversarial Attacks}.
\newblock In Arindam Banerjee and Kenji Fukumizu, editors, \emph{Proceedings of
  The 24th International Conference on Artificial Intelligence and Statistics},
  volume 130 of \emph{Proceedings of Machine Learning Research}, pages
  991--999. PMLR, 13--15 Apr 2021.

\bibitem[Shariff and Sheffet(2018)]{SS18}
Roshan Shariff and Or~Sheffet.
\newblock Differentially private contextual linear bandits.
\newblock In S.~Bengio, H.~Wallach, H.~Larochelle, K.~Grauman, N.~Cesa-Bianchi,
  and R.~Garnett, editors, \emph{Advances in Neural Information Processing
  Systems}, volume~31. Curran Associates, Inc., 2018.

\bibitem[{Hanna} et~al.(2022){Hanna}, {Girgis}, {Fragouli}, and
  {Diggavi}]{HGFD22}
Osama~A. {Hanna}, Antonious~M. {Girgis}, Christina {Fragouli}, and Suhas
  {Diggavi}.
\newblock {Differentially Private Stochastic Linear Bandits: (Almost) for
  Free}.
\newblock \emph{arXiv e-prints}, art. arXiv:2207.03445, July 2022.

\bibitem[Zheng et~al.(2020)Zheng, Cai, Huang, Li, and Wang]{ZCH+20}
Kai Zheng, Tianle Cai, Weiran Huang, Zhenguo Li, and Liwei Wang.
\newblock Locally differentially private (contextual) bandits learning.
\newblock In H.~Larochelle, M.~Ranzato, R.~Hadsell, M.F. Balcan, and H.~Lin,
  editors, \emph{Advances in Neural Information Processing Systems}, volume~33,
  pages 12300--12310. Curran Associates, Inc., 2020.

\bibitem[Han et~al.(2021)Han, Liang, Wang, and Zhang]{HLWZ21}
Yuxuan Han, Zhipeng Liang, Yang Wang, and Jiheng Zhang.
\newblock Generalized linear bandits with local differential privacy.
\newblock In M.~Ranzato, A.~Beygelzimer, Y.~Dauphin, P.S. Liang, and J.~Wortman
  Vaughan, editors, \emph{Advances in Neural Information Processing Systems},
  volume~34, pages 26511--26522. Curran Associates, Inc., 2021.

\bibitem[Bittau et~al.(2017)Bittau, Erlingsson, Maniatis, Mironov, Raghunathan,
  Lie, Rudominer, Kode, Tinnes, and Seefeld]{BEM+17}
Andrea Bittau, \'{U}lfar Erlingsson, Petros Maniatis, Ilya Mironov, Ananth
  Raghunathan, David Lie, Mitch Rudominer, Ushasree Kode, Julien Tinnes, and
  Bernhard Seefeld.
\newblock {Prochlo: Strong Privacy for Analytics in the Crowd}.
\newblock In \emph{Proceedings of the 26th Symposium on Operating Systems
  Principles}, SOSP '17, page 441–459, New York, NY, USA, 2017. Association
  for Computing Machinery.
\newblock ISBN 9781450350853.
\newblock \doi{10.1145/3132747.3132769}.

\bibitem[{Cheu}(2021)]{Cheu21}
Albert {Cheu}.
\newblock {Differential Privacy in the Shuffle Model: A Survey of Separations}.
\newblock \emph{arXiv e-prints}, art. arXiv:2107.11839, 2021.

\bibitem[Tenenbaum et~al.(2021)Tenenbaum, Kaplan, Mansour, and Stemmer]{TKMS21}
Jay Tenenbaum, Haim Kaplan, Yishay Mansour, and Uri Stemmer.
\newblock {Differentially Private Multi-Armed Bandits in the Shuffle Model}.
\newblock In M.~Ranzato, A.~Beygelzimer, Y.~Dauphin, P.S. Liang, and J.~Wortman
  Vaughan, editors, \emph{Advances in Neural Information Processing Systems},
  volume~34, pages 24956--24967. Curran Associates, Inc., 2021.

\bibitem[Chowdhury and Zhou(2022)]{CZ22a}
Sayak~Ray Chowdhury and Xingyu Zhou.
\newblock Shuffle private linear contextual bandits.
\newblock \emph{arXiv e-prints}, art. arXiv:2202.05567, 2022.

\bibitem[Garcelon et~al.(2022)Garcelon, Chaudhuri, Perchet, and
  Pirotta]{GCPP22}
Evrard Garcelon, Kamalika Chaudhuri, Vianney Perchet, and Matteo Pirotta.
\newblock Privacy amplification via shuffling for linear contextual bandits.
\newblock In Sanjoy Dasgupta and Nika Haghtalab, editors, \emph{Proceedings of
  The 33rd International Conference on Algorithmic Learning Theory}, volume 167
  of \emph{Proceedings of Machine Learning Research}, pages 381--407. PMLR, 29
  Mar--01 Apr 2022.

\bibitem[Lykouris et~al.(2018)Lykouris, Mirrokni, and Paes~Leme]{LMP18}
Thodoris Lykouris, Vahab Mirrokni, and Renato Paes~Leme.
\newblock Stochastic bandits robust to adversarial corruptions.
\newblock In \emph{Proceedings of the 50th Annual ACM SIGACT Symposium on
  Theory of Computing}, STOC 2018, page 114–122, New York, NY, USA, 2018.
  Association for Computing Machinery.
\newblock ISBN 9781450355599.
\newblock \doi{10.1145/3188745.3188918}.

\bibitem[Gupta et~al.(2019)Gupta, Koren, and Talwar]{GKT19}
Anupam Gupta, Tomer Koren, and Kunal Talwar.
\newblock Better algorithms for stochastic bandits with adversarial
  corruptions.
\newblock In Alina Beygelzimer and Daniel Hsu, editors, \emph{Proceedings of
  the Thirty-Second Conference on Learning Theory}, volume~99 of
  \emph{Proceedings of Machine Learning Research}, pages 1562--1578. PMLR,
  25--28 Jun 2019.

\bibitem[Jun et~al.(2018)Jun, Li, Ma, and Zhu]{JLMZ18}
Kwang-Sung Jun, Lihong Li, Yuzhe Ma, and Jerry Zhu.
\newblock Adversarial attacks on stochastic bandits.
\newblock In S.~Bengio, H.~Wallach, H.~Larochelle, K.~Grauman, N.~Cesa-Bianchi,
  and R.~Garnett, editors, \emph{Advances in Neural Information Processing
  Systems}, volume~31. Curran Associates, Inc., 2018.

\bibitem[Liu and Shroff(2019)]{LS19}
Fang Liu and Ness Shroff.
\newblock Data poisoning attacks on stochastic bandits.
\newblock In Kamalika Chaudhuri and Ruslan Salakhutdinov, editors,
  \emph{Proceedings of the 36th International Conference on Machine Learning},
  volume~97 of \emph{Proceedings of Machine Learning Research}, pages
  4042--4050. PMLR, 09--15 Jun 2019.

\bibitem[Garcelon et~al.(2020)Garcelon, Roziere, Meunier, Tarbouriech, Teytaud,
  Lazaric, and Pirotta]{GRM+20}
Evrard Garcelon, Baptiste Roziere, Laurent Meunier, Jean Tarbouriech, Olivier
  Teytaud, Alessandro Lazaric, and Matteo Pirotta.
\newblock Adversarial attacks on linear contextual bandits.
\newblock In H.~Larochelle, M.~Ranzato, R.~Hadsell, M.F. Balcan, and H.~Lin,
  editors, \emph{Advances in Neural Information Processing Systems}, volume~33,
  pages 14362--14373. Curran Associates, Inc., 2020.

\bibitem[{Li} et~al.(2019){Li}, {Lou}, and {Shan}]{LLS19}
Yingkai {Li}, Edmund~Y. {Lou}, and Liren {Shan}.
\newblock {Stochastic Linear Optimization with Adversarial Corruption}.
\newblock \emph{arXiv e-prints}, art. arXiv:1909.02109, 2019.

\bibitem[Chen et~al.(2022)Chen, Koehler, Moitra, and Yau]{CMY22}
Sitan Chen, Frederic Koehler, Ankur Moitra, and Morris Yau.
\newblock Online and distribution-free robustness: Regression and contextual
  bandits with huber contamination.
\newblock In \emph{2021 IEEE 62nd Annual Symposium on Foundations of Computer
  Science (FOCS)}, pages 684--695. IEEE, 2022.

\bibitem[{Diakonikolas} and {Kane}(2019)]{DK19}
Ilias {Diakonikolas} and Daniel~M. {Kane}.
\newblock {Recent Advances in Algorithmic High-Dimensional Robust Statistics}.
\newblock \emph{arXiv e-prints}, art. arXiv:1911.05911, November 2019.

\bibitem[Prasad et~al.(2019)Prasad, Balakrishnan, and Ravikumar]{PBR19}
Adarsh Prasad, Sivaraman Balakrishnan, and Pradeep Ravikumar.
\newblock A unified approach to robust mean estimation.
\newblock \emph{arXiv preprint arXiv:1907.00927}, 2019.

\bibitem[Dwork and Roth(2014)]{DR14}
Cynthia Dwork and Aaron Roth.
\newblock The algorithmic foundations of differential privacy.
\newblock \emph{Foundations and Trends{\textregistered} in Theoretical Computer
  Science}, 9\penalty0 (3--4):\penalty0 211--407, 2014.

\bibitem[Vershynin(2018)]{Vershynin18}
Roman Vershynin.
\newblock \emph{High-Dimensional Probability: An introduction with applications
  in data science}, volume~47 of \emph{Cambridge Series in Statistical and
  Probabilistic Mathematics}.
\newblock Cambridge University Press, 2018.

\bibitem[{Gross}(2011)]{Gross11}
D.~{Gross}.
\newblock Recovering low-rank matrices from few coefficients in any basis.
\newblock \emph{IEEE Transactions on Information Theory}, 57\penalty0
  (3):\penalty0 1548--1566, 2011.

\end{thebibliography}

\appendix
\newpage

\section{Proofs from~\cref{sec:robust-linear-regression}}
\label{sec:appendix:robust-linear-regression-proofs}
We will work with the empirical second moment and covariance matrices defined below:
\begin{subequations}
\begin{align}
  \widetilde{\Sigma}_{G_0} &:=
  \frac{1}{\abs{G_0}} \sum_{i \in G_0} M_n^{-1/2} y_i^2 a_i a_i^{\T} M_n^{-1/2},
                           & \Sigma_{G_0} &:= \widetilde{\Sigma}_{G_0} - \theta_{G_0} \theta_{G_0}^{\T}, \\
  \widetilde{\Sigma}_n &:=
  \frac{1}{n} \sum_{i = 1}^n M_n^{-1/2} y_i^2 a_i a_i^{\T} M_n^{-1/2},
                       & \Sigma_{n} &:= \widetilde{\Sigma}_n - \theta_n \theta_n^{\T}.
\end{align}
\end{subequations}
In addition, we will use the vector notation below:
\begin{align}
  \bm{y} = \begin{pmatrix} y_1 \\ \vdots \\ y_n \end{pmatrix},
  \quad \text{and} \quad
  \bm{\sigma} = \begin{pmatrix}
    \sigma_1 \\ \vdots \\ \sigma_n
  \end{pmatrix}.
  \label{eq:vector-notation}
\end{align}
Our guarantees will depend on the maximal weighted norm of the elements $a_i$, which
we will denote by
\begin{equation}
  \mu := \max_{i = 1, \dots, n} \norm{a_i}_{M_n^{-1}}^2. 
  \label{eq:maximal-weighted-norm}
\end{equation}

Finally, we make the following assumption:
\begin{assumption}
  \label{asm:nontrivial-corruption}
  Fix $\delta$ to be a desired failure probability. The corruption probability $\alpha$ satisfies
  \(
    \alpha \gtrsim \frac{\log(1 / \delta)}{n}.
  \)
\end{assumption}
To approximate $\theta_{G_0}$, we will first reduce the above problem to robust
mean estimation and apply the spectral filtering algorithm from the robust
statistics literature. In~\cite{PBR19}, the authors provide the following guarantee:
\begin{theorem}[{\cite[Theorem 4]{PBR19}}]
  \label{theorem:robust-mean-estimation}
  Suppose that $\lambda \geq \opnorm{\Sigma_{G_0}}$ and that the set of inliers,
  $G_0$, satisfies
  \(
    \frac{n - \abs{G_0}}{n} + \frac{\log(1 / \delta)}{n} \leq c,
  \)
  where $c$ is a dimension-independent constant. Then with probability at least
  $1 - \delta$, the spectral filtering
  algorithm for robust mean estimation terminates in at most $O((n - \abs{G_0}) +
  \log(1 / \delta))$ steps and returns an estimate $\widetilde{\theta}$ satisfying
  \begin{equation}
    \bignorm{\widetilde{\theta} - \theta_{G_0}} \leq
    C \sqrt{\lambda} \left(
      \frac{n - \abs{G_0}}{n} + \frac{\log(1 / \delta)}{n} \right)^{1/2}.
  \end{equation}
\end{theorem}
In light of~\cref{theorem:robust-mean-estimation}, we will control the
quantities involved. Before we proceed, we state the following bound for the size of
$G_0$ that we will repeatedly appeal to throughout:
\begin{lemma}
  \label{lemma:G0-size}
  Let $n \gtrsim \frac{\log(1 / \delta)}{\alpha}$. Then with probability at
  least $1 - \delta$, we have
  \begin{equation}
    \abs{\frac{\abs{G_0}}{n} - (1 - \alpha)} \leq \sqrt{\frac{\alpha \log(1 / \delta)}{n}}
  \end{equation}
\end{lemma}
\begin{proof}
  Let $S_{n} = \sum_{i = 1}^n \bm{1}\set{i \notin G_0}$, which is a sum of i.i.d.
  Bernoulli random variables with parameter $\alpha$. From a Chernoff bound~\citep[Exercise 2.3.5]{Vershynin18},
  it follows that
  \begin{equation}
    \prob{\abs{S_n - n \alpha} \geq \sqrt{n \alpha \log(1 / \delta)}} \leq \delta,
    \quad \text{for $\delta \in (0, 1]$}.
    \label{eq:chernoff}
  \end{equation}
  The claim follows after dividing both sides of the inequality by $n$:
  \begin{equation*}
    \frac{\abs{G_0}}{n} = 1  - \frac{S_n}{n} \in (1 - \alpha) \pm \sqrt{\frac{\alpha \log(1 / \delta)}{n}}.
  \end{equation*}
\end{proof}

\subsection{Controlling the empirical mean} \label{subsec:robust-linreg-empmean}
We now control the deviation of $\theta_{G_0}$ from the mean of the dataset
absent any corruptions.
\begin{lemma}
  \label{lemma:robust-linreg-empmean-bound}
  With probability at least $1 - \delta$, we have
  \begin{equation}
    \norm{\theta_{G_0} - \frac{n(1 - \alpha)}{\abs{G_0}} \theta_n} \lesssim
    \frac{\norm{\bm{y}}}{\abs{G_0}} \sqrt{\mu \alpha(1 - \alpha) \log(1/\delta)}.
  \end{equation}
\end{lemma}
\begin{proof}
  Define the following collection of random variables:
  \begin{equation}
    Q_i := (Z_i - \expec{Z_i}) M_n^{-1/2} y_i a_i, \quad \text{with} \; \;
    Z_i = \bm{1}\set{i \in G_0}.
  \end{equation}
  Clearly, we have $\expec{Q_i} = 0$. At the same time,
  \[
    \expec{\norm{Q_i}^2} = \mathrm{Var}(Z_i) y_i^2 \norm{a_i}_{M_n^{-1}}^2
    \leq \alpha(1 - \alpha) y_i^2 \mu.
  \]
  Applying the vector Bernstein inequality~\citep[Theorem 12]{Gross11}, we obtain
  \begin{align*}
    \prob{\Bignorm{\sum_{i = 1}^n Q_i} \geq \norm{\bm{y}} \sqrt{\mu \alpha(1 - \alpha)} + t}
    \leq \expfun{-\frac{t^2}{\mu \alpha(1 - \alpha) \norm{\bm{y}}^2}}.
  \end{align*}
  Consequently, we may set
  \(
    t = \norm{\bm{y}} \sqrt{\mu \alpha(1 - \alpha) \log(1 / \delta)}
  \) to obtain the claimed probability. Finally, we note that
  \[
    \sum_{i = 1}^n Q_i = \sum_{i \in G_0} M_n^{-1/2} y_i a_i - 
    (1 - \alpha) \sum_{i = 1}^n M_n^{-1/2} y_i a_i =
    \abs{G_0} \theta_{G_0} -  n (1 - \alpha) \theta_{n}.
  \]
\end{proof}

\subsection{Putting everything together}
We now combine the bounds from~\cref{subsec:robust-linreg-empmean} and~\cref{theorem:robust-mean-estimation}. We first note that
\begin{align*}
  \Sigma_{G_0} = \widetilde{\Sigma}_{G_0} - \theta_{G_0} \theta_{G_0}^{\T}
  &\preceq \widetilde{\Sigma}_{G_0} \\
               &= \frac{1}{\abs{G_0}} \sum_{i \in G_0} y_i^2 M_n^{-1/2} a_i a_i^{\T} M_n^{-1/2} \\
               &\preceq
               \frac{1}{\abs{G_0}} \sum_{i \in G_0} y_i^2 \bigopnorm{M_n^{-1/2} a_i a_i^{\T} M_n^{-1/2}} I_{d} \\
               &\preceq
               \frac{1}{\abs{G_0}} \sum_{i \in G_0} y_i^2 \bignorm{M_n^{-1/2} a_i}^2 I_{d} \\
               &\preceq
               \frac{1}{\abs{G_0}} \sum_{i \in G_0} y_i^2 \norm{a_i}_{M_n^{-1}}^2 I_{d},
\end{align*}
which implies that the spectral norm of $\Sigma_{G_0}$ is bounded from above by
\begin{equation}
  \opnorm{\Sigma_{G_0}} \leq
  \frac{\norm{\bm{y}_{G_0}}^2}{\abs{G_0}} \cdot \max_{i} \norm{a_i}_{M_n^{-1}}^2 \leq
  \frac{\mu \norm{\bm{y}}^2}{\abs{G_0}}.
\end{equation}
At the same time, we appeal to~\cref{lemma:G0-size} to deduce that
\[
  \abs{G_0} \geq (1 - \alpha)n - 3 \sqrt{n \log(1 / \delta)}
  \geq \frac{(1 - \alpha) n}{2},
  \quad \text{for $n \geq \frac{18\log(1 / \delta)}{(1 - \alpha)^2}$}.
\]
Consequently, we can replace the previous upper bound with
\(
  \opnorm{\Sigma_{G_0}} \leq \frac{2\mu \norm{\bm{y}}^2}{n (1 - \alpha)}.
\)

We now appeal to~\cref{theorem:robust-mean-estimation}. Note that
Lemma~\ref{lemma:G0-size} yields
\[
  \frac{n - \abs{G_0}}{n} = 1 - \frac{\abs{G_0}}{n} \leq 1 - (1 - \alpha) + \sqrt{\frac{\alpha \log(1 / \delta)}{n}}
  = \alpha + \sqrt{\frac{\alpha \log(1 / \delta)}{n}}.
\]
Therefore, the estimate $\widetilde{\theta}$ computed by the spectral filtering
algorithm satisfies
\begin{equation}
  \bignorm{\widetilde{\theta} - \theta_{G_0}} \lesssim
  \norm{\bm{y}} \sqrt{\frac{2 \mu}{n(1 - \alpha)}} \left(
    \alpha + \sqrt{\frac{\alpha \log(1 / \delta)}{n}} + \frac{\log(1 / \delta)}{n}
  \right)^{1/2}.
  \label{eq:robust-mean-est-error-bound}
\end{equation}
Taking a union bound over Lemmas~\ref{lemma:G0-size} and~\ref{lemma:robust-linreg-empmean-bound},
we deduce that~\eqref{eq:robust-mean-est-error-bound} holds with probability
at least $1 - 2\delta$.

\subsection{Application to phased elimination}
Let $\theta_{\mathsf{LS}} := M_n^{-1}\sum_{i = 1}^n y_i a_i$ denote the least
squares solution from an approximate G-optimal design, and define
\begin{subequations}
  \begin{align}
    \bar{\theta}_{G_0} &= M^{-1} \sum_{i \in G_0} y_i a_i = M^{-1/2} \abs{G_0} \theta_{G_0}, \\
    \bar{\theta} &= n M^{-1/2} \widetilde{\theta}.
  \end{align}
\end{subequations}
Note that $\bar{\theta}$ can be computed from the output of~\cref{alg:spectral-filtering},
while $\bar{\theta}_{G_0}$ only serves for the analysis.
With these at hand, we have the following decomposition:
\begin{align}
  \ip{a, \bar{\theta} - \theta^{\star}} &=
  \ip{a, \bar{\theta} - \bar{\theta}_{G_0}} +
  \ip{a, \bar{\theta}_{G_0} - \theta_{\mathsf{LS}}} + \ip{a, \theta_{\mathsf{LS}} - \theta^{\star}}
  \label{eq:decomp-1}
\end{align}
In what follows, we bound each term in~\eqref{eq:decomp-1} separately.
\subsubsection{Bounding the first term in~\eqref{eq:decomp-1}}
The first term in~\eqref{eq:decomp-1} is equal to
\begin{align}
  \ip{M^{-1/2} a, M^{1/2}(\bar{\theta} - \bar{\theta}_{G_0})} &=
  \ip{M^{-1/2} a, n \widetilde{\theta} - \abs{G_0} \theta_{G_0}} \notag \\
                                                              &=
  \ip{M^{-1/2} a, n(\widetilde{\theta} - \theta_{G_0})} +
  ( n - \abs{G_0} ) \ip{M^{-1/2} a, \theta_{G_0}} \notag \\
                                                              &\leq
  \norm{a}_{M^{-1}} \bignorm{n(\widetilde{\theta} - \theta_{G_0})} +
  ( n - \abs{G_0} ) \ip{M^{-1/2} a, \theta_{G_0}}
  \label{eq:decomp-1-inter-1}
\end{align}
In particular, the second term in~\eqref{eq:decomp-1-inter-1} is given by
\begin{align}
  \ip{M^{-1/2} a, \theta_{G_0}} &=
  \frac{1}{\abs{G_0}} \ip{M^{-1/2} a, M^{-1/2} \sum_{i \in G_0} y_i a_i} \notag \\
                                &\leq
  \frac{1}{\abs{G_0}} \norm{a}_{M^{-1}}
  \Bignorm{\sum_{i \in G_0} y_i a_i}_{M^{-1}} \notag \\
                                &\leq
  \frac{1}{\abs{G_0}} \norm{a}_{M^{-1}}
  \Bignorm{\big(\sum_{i \in G_0} a_i a_i^{\T}\big)^{-1/2} \sum_{i \in G_0} y_i a_i},
  \label{eq:decomp-1-inter-2}
\end{align}
using the fact that $\sum_{i \in G_0} a_i a_i^{\T} \preceq \sum_{i=1}^n a_i a_i^{\T}$.
Let $A_{G_0}$ be a matrix whose rows are the vectors $\set{a_i \mid i \in G_0}$.
We have
\[
  \sum_{i \in G_0} a_i a_i^{\T} = A_{G_0}^{\T} A_{G_0}, \quad
  \text{and} \quad
  \sum_{i \in G_0} y_i a_i = A_{G_0}^{\T} \bm{y}_{G_0}.
\]
Letting $A_{G_0} = U \Sigma V^{\T}$ denote the economic SVD of $A_{G_0}$, we thus have
\begin{equation}
  \Bignorm{\big(\sum_{i \in G_0} a_i a_i^{\T}\big)^{-1/2} \sum_{i \in G_0} y_i a_i} =
  \bignorm{(A_{G_0}^{\T} A_{G_0})^{-1/2} A_{G_0}^{\T} \bm{y}_{G_0}} =
  \bignorm{V \Sigma^{-1} V^{\T} V \Sigma U^{\T} \bm{y}_{G_0}} \leq
  \norm{\bm{y}}.
  \label{eq:decomp-1-norm-bound}
\end{equation}
Plugging~\cref{eq:decomp-1-norm-bound} into~\cref{eq:decomp-1-inter-2} and the result
into~\cref{eq:decomp-1-inter-1}, we obtain
\begin{equation*}
  \ip{M^{-1/2} a, M^{1/2}(\bar{\theta} - \bar{\theta}_{G_0})} \leq
  \norm{a}_{M^{-1}} \left(
    \bignorm{n(\widetilde{\theta} - \theta_{G_0})} +
    \frac{n - \abs{G_0}}{\abs{G_0}} \norm{\bm{y}}
  \right)
\end{equation*}
Using~\cref{eq:robust-mean-est-error-bound}, the bound $\norm{a}_{M^{-1}} \leq
\sqrt{\mu}$, and~\cref{lemma:G0-size} with $\alpha \gtrsim
\frac{\log(1 / \delta)}{n}$, the above becomes:
\begin{equation}
  \ip{M^{-1/2} a, M^{1/2}(\bar{\theta} - \bar{\theta}_{G_0})} \lesssim
  \mu\norm{\bm{y}} \sqrt{n}
    \left( \alpha + \frac{\log(1 / \delta)}{n} \right)^{1/2}
  \label{eq:conf-interval-term-1}
\end{equation}

\subsubsection{Bounding the second term in~\eqref{eq:decomp-1}}
Recall that $\bar{\theta}_{G_0} = M^{-1/2} \abs{G_0} \theta_{G_0}$.
We further decompose the second term in~\eqref{eq:decomp-1} into
\begin{align}
  \ip{a, \bar{\theta}_{G_0} - \theta_{\mathsf{LS}}} &=
  \ip{M^{-1/2} a, M^{1/2}(\bar{\theta}_{G_0} - \theta_{\mathsf{LS}})} \notag \\
                                                    &=
  \ip{M^{-1/2} a, M^{-1/2} \left( \sum_{i \in G_0} y_i a_i - \sum_{i = 1}^n y_i a_i \right)}
  \notag \\
                                                    &=
  \ip{M^{-1/2} a, \abs{G_0}\left( \theta_{G_0} - \frac{n}{\abs{G_0}} \theta_n\right)}
  \notag \\
                                                    &=
                                                    \ip{M^{-1/2} a, \abs{G_0}\left( \theta_{G_0} - \frac{n(1 - \alpha)}{\abs{G_0}} \theta_n\right)} + \ip{M^{-1/2} a, n \alpha \theta_n}
  \label{eq:decomp-2}
\end{align}
The first term in~\eqref{eq:decomp-2} can be upper bounded using~\cref{lemma:robust-linreg-empmean-bound}. Indeed,
\begin{align*}
  \ip{M^{-1/2} a, \abs{G_0} \left(\theta_{G_0} - \frac{n(1 - \alpha)}{\abs{G_0}} \theta_n\right)} &\leq
  \norm{a}_{M^{-1}} \norm{\bm{y}} \sqrt{\mu \alpha \log(1 / \delta)}
  \leq
  \mu \norm{\bm{y}} \sqrt{\alpha \log(1 / \delta)}.
\end{align*}
We now simplify the second term in~\eqref{eq:decomp-2}. With
$e_i = r_i - \ip{a_i, \theta^{\star}}$, we obtain
\begin{align*}
  \ip{M^{-1/2} a, n \alpha \theta_n} &=
  \alpha \ip{a, \sum_{i = 1}^n M^{-1} y_i a_i} \\
                                     &=
                             \alpha \ip{a, \sum_{i = 1}^n M^{-1} a_i (\ip{a_i, \theta^{\star}} + e_i)} \\
                                     &=
                                     \alpha \ip{a, M^{-1} \left(\sum_{i=1}^n a_i a_i^{\T}\right) \theta^{\star}} +
  \alpha \sum_{i = 1}^n \ip{a, M^{-1} a_i} e_i \\
                                     &=
                                     \alpha \ip{a, \theta^{\star}} + \alpha \sum_{i = 1}^n \ip{a, M^{-1} a_i} e_i.
\end{align*}
Since $\max_{a \in \cA} \abs{\ip{a, \theta^{\star}}} \leq 1$, combining the two bounds above yields
\begin{align}
  \ip{a, \bar{\theta}_{G_0} - \theta_{\mathsf{LS}}} &\lesssim
  \mu \norm{\bm{y}} \sqrt{\alpha \log(1 / \delta)} +
  \alpha \left( 1 + \sum_{i=1}^n \ip{a, M^{-1} a_i} e_i \right).
  \label{eq:conf-interval-term-2}
\end{align}
\subsubsection{Bounding the third term in~\eqref{eq:decomp-1}}
The last term is straightforward to bound. Let $e_i = y_i - \ip{a_i, \theta^{\star}}$
and note that
\begin{equation}
  \theta_{\mathsf{LS}} - \theta^{\star} = M^{-1} \sum_{i = 1}^n y_i a_i - \theta^{\star} =
  M^{-1} \sum_{i = 1}^n a_i (\ip{a_i, \theta^{\star}} + e_i) - \theta^{\star} =
  M^{-1} \sum_{i = 1}^n a_i e_i.
  \label{eq:conf-interval-term-3}
\end{equation}

\subsubsection{Putting everything together}
\label{sec:appendix:confidence-interval}
Combining~\cref{eq:conf-interval-term-1,eq:conf-interval-term-2,eq:conf-interval-term-3}
yields the following robust confidence intervals:
\begin{equation}
  \abs{\ip{a, \bar{\theta} - \theta^{\star}}} \lesssim
  \mu \norm{\bm{y}} \left[
    \sqrt{n} \left(\alpha + \frac{\log(1 / \delta)}{n}\right)^{1/2}
  + \sqrt{\alpha \log(1 / \delta)}
  \right] +
  \sum_{i = 1}^n e_i \ip{a, M^{-1} a_i} + \alpha.
  \label{eq:conf-interval-all-terms}
\end{equation}

\section{Missing proofs from~\cref{sec:dp-analysis}}

\subsection{Missing proofs from~\cref{sec:M1-DP}}

\subsubsection{Proof of~\cref{lemma:M1-error-contrib}}
\begin{proof}
  We write $e_i = \cM(r_i) - \ip{a_i, \theta^{\star}} = \eta_i + \xi_i, \; \;
  \eta_i \sim \mathrm{SubG}(1), \; \xi_i \sim \mathsf{Lap}\left(\frac{2}{\epspriv}\right)$.
  Now, define the random variables
  \begin{align*}
    X_i &:= \eta_i \ip{a, M^{-1} a_i}; \qquad
    Y_i := \xi_i \ip{a, M^{-1} a_i}.
  \end{align*}
  The family $\set{X_i}$ is subgaussian with \(\norm{X_i}_{\psi_2} \leq
  \abs{\ip{a, M^{-1} a_i}} \). Consequently,
  \begin{align*}
    \sum_{i = 1}^n \norm{X_i}_{\psi_2}^2 &\leq
    \sum_{i = 1}^n \ip{a, M^{-1} a_i}^2 \\
                                         &=
    \sum_{i = 1}^n \mathsf{Tr}(a^{\T} M^{-1} a_i a_i^{\T} M^{-1} a) \\
                                         &=
    \ip{M^{-1} a, \sum_{i = 1}^n a_i a_i^{\T} M^{-1} a} \\
                                         &=
    \ip{a, M^{-1} a} \\
                                         &=
    \norm{a}_{M^{-1}}^2.
  \end{align*}
  Therefore, applying the Hoeffding inequality~\citep[Theorem 2.6.2]{Vershynin18} yields:
  \begin{equation}
    \prob{\abs{\sum_{i=1}^n \eta_i \ip{a, M^{-1} a_i}} \geq c_1 \norm{a}_{M^{-1}} \sqrt{\log(1 / \delta)}} \leq
    \delta
    \label{eq:conf-interval-subgaussian-bound}
  \end{equation}
  On the other hand, when $\xi_i \sim \mathsf{Lap}(2 / \epspriv)$, we have
  the Bernstein-style bound
  \begin{align*}
    \expec{e^{\lambda \sum_{i = 1}^n \xi_i \ip{a, M^{-1} a_i}}} &=
    \prod_{i = 1}^n \expec{\exp\left(\lambda \xi_i \ip{a, M^{-1} a_i}\right)} \\
    &\leq
    \prod_{i = 1}^n \exp\left(\frac{\lambda^2 \ip{a, M^{-1} a_i}^2}{2\epspriv ^2}\right), \quad
    \forall \lambda \in \left(0, \frac{b}{\|\alpha\|_{\infty}}\right],
  \end{align*}
  using~\citep[Proposition 2.7.1(e)]{Vershynin18} in the last step. Collecting terms
  we obtain
  \[
    \prod_{i = 1}^n \expfun{
      \frac{\lambda^2 \ip{a, M^{-1} a_i}^2}{2 \epspriv^2}
    } =
    \expfun{
    \lambda^2 \frac{\sum_{i = 1}^n \ip{a, M^{-1} a_i}^2}{\epspriv^2}} \leq
    \expfun{\lambda^2 c_1 \left( \frac{\norm{a}_{M^{-1}}}{\epspriv} \right)^2}.
  \]
  Now, appealing to~\citep[Proposition 2.7.1(a)]{Vershynin18}, we obtain the concentration
  bound
  \begin{equation}
    \prob{\Big|\sum_{i = 1}^n \xi_i \ip{a, M^{-1} a_i}\Big| \geq
    c_2 \frac{\norm{a}_{M^{-1}} \log(1 / \delta)}{\epspriv}} \leq \delta.
    \label{eq:conf-interval-laplace-bound}
  \end{equation}
  Combining the two bounds yields the result.
\end{proof}

\subsubsection{Proof of~\cref{lemma:M1-y-norm}}
\begin{proof}
  Let $\pi$ denote an approximate G-optimal design in the sense
  of~\cref{def:approx-g-opt-design}. We have
  \begin{align*}
    M &= \sum_{i = 1}^n a_i a_i^{\T} = \sum_{a \in \supp(\pi)} n_{a} aa^{\T}
    = n \cdot \sum_{a \in \supp(\pi)} \pi(a) aa^{\T} = n M(\pi).
  \end{align*}
  Consequently, we have the inequality
  \begin{align*}
    \norm{a}_{M^{-1}}^2 &= \ip{a, M^{-1} a} = \ip{a, (n M(\pi))^{-1} a}
    = \frac{1}{n} \ip{a, M^{-1}(\pi) a} = \frac{\norm{a}^2_{M^{-1}(\pi)}}{n}
    \leq \frac{2d}{n},
  \end{align*}
  using the fact that $\pi$ is an approximate G-optimal design
  in the last inequality.
\end{proof}

\subsubsection{Proof of~\cref{lemma:M1-y-norm}}
\begin{proof}
  With $\bm{y} = \bmx{y_1 & \dots & y_n}^{\T}$, we have $\norm{\bm{y}} \leq \sqrt{n} \norm{\bm{y}}_{\infty}$.
  To control the latter, we note
  \begin{align*}
    \max_{i}\abs{\ip{a_i, \theta^{\star}} + \eta_i + \xi_i} \leq
    \max_{i} \set{ \abs{\ip{a_i,\theta^{\star}}} + \abs{\eta_i} + \abs{\xi_i} }
                                                            \leq
    1 + \max_{i} \abs{\eta_i} + \max_{i} \abs{\xi_i}.
  \end{align*}
  Since $\eta_i \sim \mathrm{SubG}(1)$, standard concentration inequalities for
  subgaussian maxima yield
  \begin{equation}
    \prob{\max_{i} \abs{\eta_i} \geq C \sqrt{\log(n / \delta)}}
    \leq \delta.
    \label{eq:eta-max-norm-bound}
  \end{equation}
  Similarly, $\xi_i$ are subexponential with parameter $2/\epspriv$.
  By a union bound and~\citep[Proposition 2.7.1]{Vershynin18},
  \begin{align*}
    \prob{\max_{i} \abs{\xi_i} \geq t} &\leq
    \sum_{i = 1}^n \prob{\abs{\xi_i} \geq t}
                                       \leq
    n \expfun{-\min\set{\frac{\epspriv^2 t^2}{8}, \frac{\epspriv t}{4}}}
  \end{align*}
  Setting $t := \frac{4 \log(n / \delta)}{\epspriv}$ above yields
  $\max_{i} \abs{\xi_i} \leq \frac{4 \log(n / \delta)}{\epspriv}$ with probability
  at least $1 - \delta$.

  Finally, taking a union bound and relabelling yields the result.
\end{proof}

\subsubsection{Proof of~\cref{theorem:M1-regret}}
\begin{proof}
We perform a regret analysis under the LDP model~\ref{item:M1}. Using~\cref{prop:robust-confidence-intervals,lemma:M1-error-contrib,lemma:M1-y-norm} and~\cref{asm:nontrivial-corruption},
writing $\mu := \max_{a \in \cA} \norm{a}_{M^{-1}}^2$, we have
\begin{align}
  \big|\langle a, \widetilde{\theta} - \theta^{\star} \rangle\big|& \lesssim
  \begin{aligned}[t]
    & \mu \left( n \sqrt{\alpha} + \sqrt{\alpha n \log(1 / \delta)} \right)
    \left(1 + \sqrt{\log(n / \delta)} + \frac{\log(n/\delta)}{\epspriv}\right) \\
    & + \sqrt{\mu \log(1 / \delta)} \left(1 + \frac{\sqrt{\log(1/\delta)}}{\epspriv}\right) + \alpha,
  \end{aligned}
  \label{eq:simple-conf-interval}
\end{align}
for any fixed $a$ with probability at least $1 - \delta$ by suitably adjusting
constants.
In particular, when $a_1, \dots, a_n$ are drawn from an approximate
G-optimal design,~\cref{lemma:M1-y-norm} implies that
\begin{equation}
  \mu \equiv \max_{a \in \cA} \norm{a}_{M_n^{-1}}^2 \leq \frac{2d}{n},
  \label{eq:norm-bound-g-optimal-M1}
\end{equation}
so the bound in~\eqref{eq:simple-conf-interval} can be written as
\begin{equation}
  \big|\langle a, \widetilde{\theta} - \theta^{\star} \rangle\big|
  \lesssim
  \begin{aligned}[t]
  & d \left( \sqrt{\alpha} + \sqrt{\frac{\alpha \log(1 / \delta)}{n}} \right)
  \left(1 + \sqrt{\log(n / \delta)} + \frac{\log(n/\delta)}{\epspriv}\right) \\
  & + \sqrt{\frac{d \log(1 / \delta)}{n}} \left(1 + \frac{\sqrt{\log(1/\delta)}}{\epspriv}\right),
  \end{aligned}
  \label{eq:simple-conf-interval-1}
\end{equation}
By standard arguments (see, e.g., the proof of~\citep[Theorem 5.1]{EKMM21}), we may focus
on bounding the regret conditioned on the ``good'' event where all the invocations
to the coreset construction and robust filtering algorithms succeed. This requires us to
choose failure probability $\delta$ proportional to $\delta' / (KT^2)$, where $T$ is the
number of rounds, $K$ is the size of the action space, and $\delta'$ is an overall desired
failure probability. To ease notation, we relabel $\delta$ in this manner below.

Now, recalling the width of the confidence interval
\[
  \gamma_i := \sqrt{d}\left(\sqrt{\log(q^i / \delta)} + \frac{\log(q^i / \delta)}{\epspriv}\right)
  (\sqrt{\alpha} + \alpha \sqrt{d}) +
  \alpha + \sqrt{\frac{d \log(1 / \delta)}{q^i}}\left(
    1 + \frac{\sqrt{\log(1 / \delta)}}{\epspriv}
  \right),
\]
we have the following expression for the regret:
\begin{align}
    \mathsf{Regret} &= \sum_{i = 1}^B (\text{arms pulled}) \times (\text{instantaneous regret}) \notag \\
    &\leq
    \sum_{i = 1}^B q^i 4 \gamma_{i - 1} \notag \\
    &\lesssim
    \begin{aligned}[t]
    & \sum_{i = 1}^B q^i
    \sqrt{\frac{d \log(1 / \delta)}{q^{i - 1}}}
    \left(1 + \frac{\sqrt{\log(1 / \delta)}}{\epspriv}\right) \\
    & +
    \sqrt{d}(\sqrt{\alpha} + \alpha\sqrt{d})
    \sum_{i = 1}^B q^i \left(
      \sqrt{\log(q^{i-1} / \delta)} + \frac{\log(q^{i-1} / \delta)}{\epspriv}
    \right)
    \end{aligned}
    \label{eq:regret-bound-all}
\end{align}
To bound the first sum above, we notice that
\[
  \sum_{i = 1}^B q^i \sqrt{\frac{1}{q^{i-1}}} =
  q \sum_{i = 0}^{B-1} \sqrt{q^i} = q \cdot \frac{q^{B / 2} - 1}{q^{1/2} - 1}.
\]
For the second sum, we first bound
\(
  \log(q^{i-1} / \delta) \leq \log(T^{\frac{B-1}{B}} / \delta) \leq \log(T / \delta),
\)
followed by
\[
  \sum_{i=1}^B q^i \left(
    \sqrt{\log(q^{i-1} / \delta)} + \frac{\log(q^{i-1}/\delta)}{\epspriv}
  \right) \leq
  \left( \sqrt{\log(T / \delta)} + \frac{\log(T / \delta)}{\epspriv} \right) T.
\]
Finally, we note that when $B \geq \log(T)$ we have
$\frac{q^{B/2} - 1}{q^{1/2} - 1} \lesssim \sqrt{T}$ and $q=T^{1/B} \leq e$. Therefore,
\begin{align*}
  \mathsf{Regret} \lesssim
  \sqrt{T d \log(1 / \delta)} \left(1 + \frac{\sqrt{\log (1 / \delta)}}{\epspriv} \right)
  + T \max\set{\sqrt{\alpha d}, \alpha d} \left(
    \sqrt{\log(T / \delta)} + \frac{\log(T / \delta)}{\epspriv}
  \right).
\end{align*}
\end{proof}

\subsection{Missing proofs from~\cref{sec:M2-DP}}
\subsubsection{Proof of~\cref{lemma:M2-error-contrib}}
\begin{proof}
  We have $e_{v} = \cM(r_{v}) - \ip{v, \theta^{\star}} = \eta_{v} + \xi_{v}$, where
  $\eta_v \sim \mathrm{SubG}(1 / n_{a})$ and $\xi_v \sim \mathsf{Lap}(2 / n_{a} \epspriv)$.
  Therefore,
  \[
    \eta_{v} + \xi_{v} \overset{(d)}{=} \frac{1}{\sqrt{n_{a}}} \tilde{\eta}_v + \frac{1}{n_{a}} \tilde{\xi}_v,
    \quad \tilde{\eta}_v \sim \mathrm{SubG}(1), \; \tilde{\xi}_v \sim \mathsf{Lap}(2 / \epspriv).
  \]
  Consequently, we may trace the proof of~\cref{lemma:M1-error-contrib} to arrive at
  \begin{equation}
    \sum_{v \in \cH} \eta_v \ip{a, M^{-1} v} \lesssim
    \norm{a}_{M^{-1}} \sqrt{\log(1 / \delta)} \left(
      \frac{c_1}{\sqrt{n_a}} + \frac{c_2 \sqrt{\log(1/\delta)}}{n_{a} \epspriv}
    \right).
  \end{equation}
  This completes the proof after noticing that $n_{a} \geq \nu m$.
\end{proof}

\subsubsection{Proof of~\cref{lemma:norm-bound-g-optimal-M2}}
\begin{proof}
  Recall that $M = \sum_{v \in \cH} vv^{\T}$. In particular, we have
  \begin{align}
    \sum_{v \in \cH} vv^{\T} &= \sum_{v \in \mathrm{supp}(\hat{\pi})}
    \succeq \sum_{v \in \cH} \hat{\pi}(v) vv^{\T} \implies
    \norm{a}_{M^{-1}}^2 \leq \norm{a}_{M^{-1}(\hat{\pi})}^2 \leq 2d,
  \end{align}
  where the last inequality follows since ${\pi}$ is an approximate G-optimal design.
\end{proof}

\subsubsection{Proof of~\cref{lemma:M2-y-norm}}
\begin{proof}
  Let $k = \abs{\mathrm{supp}(\hat{\pi})}$ and let $y_1, \dots, y_k$ be an
  enumeration of the elements $y_{v}, \; v \in \mathrm{supp}(\hat{\pi})$.
  With $\bm{y} = \bmx{y_1 & \dots & y_k}^{\T}$, we have $\norm{\bm{y}} \leq \sqrt{k} \norm{\bm{y}}_{\infty}$.
  To control the latter, we note
  \begin{align*}
    \max_{v} \abs{\ip{v, \theta^{\star}} + \eta_v + \xi_v} \leq
    \max_{v} \set{\abs{\ip{v, \theta^{\star}}} + \abs{\eta_v} + \abs{\xi_v} }
                                                            \leq
    1 + \max_{v} \abs{\eta_v} + \max_{i} \abs{\xi_v}.
  \end{align*}
  Since $\eta_v \sim \mathrm{SubG}(1 / n_{v})$, standard concentration inequalities for
  subgaussian maxima yield
  \begin{equation}
    \prob{\max_{v} \abs{\eta_v} \geq C \sqrt{\frac{\log(k / \delta)}{n_{v}}}}
    \leq \delta.
    \label{eq:eta-max-norm-bound-M2}
  \end{equation}
  Similarly, $\xi_v$ are subexponential with parameter $\frac{2}{n_{v} \epspriv}$.
  From a union bound and~\citep[Proposition 2.7.1]{Vershynin18}, it follows that
  \begin{align*}
    \prob{\max_{v} \abs{\xi_v} \geq t} &\leq
    \sum_{v \in \mathrm{supp}(\hat{\pi})} \prob{\abs{\xi_v} \geq t} \\
                                                     &\leq
    \sum_{v \in \mathrm{supp}(\hat{\pi})} \prob{\abs{\xi_v} \geq t} \\
                                                     &\leq
    k \expfun{-\min\set{\frac{\min_{v} n^2_{v} \epspriv^2 t^2}{8}, \frac{\min_{v} n_v \epspriv t}{4}}} \\
                                                     &\leq
    k \expfun{-\min\set{\frac{\nu^2 m^2 \epspriv^2 t^2}{8}, \frac{\nu m \epspriv t}{4}}}.
  \end{align*}
  Setting $t := \frac{4 \log(k / \delta)}{\nu m \epspriv}$ above yields
  $\max_{v} \abs{\xi_v} \leq \frac{4 \log(k / \delta)}{\nu m \epspriv}$ with probability
  at least $1 - \delta$.

  Finally, taking another union bound and relabelling yields the result.
\end{proof}

\subsubsection{Proof of~\cref{theorem:M2-regret}}
\begin{proof}
We first derive an expression for the robust confidence
interval from~\cref{prop:robust-confidence-intervals} under~\ref{item:M2}.
Indeed, with probability at least $1 - \delta$, for any fixed $a \in \cA$
we have:
\begin{equation}
  \big|\langle a, \widetilde{\theta} - \theta^{\star} \rangle\big|
  \lesssim \begin{aligned}[t]
    & \sqrt{\frac{d \log(1 / \delta)}{\nu m}} \left(1 + \frac{1}{\epspriv}\sqrt{\frac{\log(1/\delta)}{\nu m}}\right) \\
    & + 2d \left(1 + \sqrt{\frac{\log (k / \delta)}{\nu m}} +
                 \frac{\log(k / \delta)}{\nu m \epspriv} \right)
            \left(\sqrt{k \alpha} + \sqrt{\alpha \log(1 / \delta)}\right) \\
    & + \alpha,
  \end{aligned}
  \label{eq:confidence-interval-M2}
\end{equation}
where $k := \abs{\mathrm{supp}({\pi})}$.
Recall we can find an approximate G-optimal design ${\pi}$ satisfying
\begin{equation}
  k := \abs{\mathrm{supp}({\pi})} \lesssim d \log \log d.
  \label{eq:sparse-design}
\end{equation}
Therefore, we may proceed with the regret analysis. Similarly to the proof of~\cref{theorem:M1-regret},
we condition on the case where all randomized algorithms and invocations to random events succeed with
high probability.

Then, with $m = q^{i}$ at round $i$, we have the following bound:
\begin{equation}
  n_{i} = \sum_{v \in \mathrm{supp}({\pi})} n_{v} =
  q^i \sum_{v} \max\set{{\pi}(v), \nu} \leq
  q^i (1 + \nu d \log \log d).
\end{equation}
In particular, we have the following property for the sum $\sum_{i} q^i$:
\begin{equation}
  \sum_{i = 1}^B q^i = \frac{1}{1 + \nu d \log \log d} \sum_{i = 1}^B n_{i} = \frac{T}{1 + \nu d \log \log d}.
  \label{eq:action-sum-M2}
\end{equation}
Consequently, the regret of the algorithm conditioned on the good event
is given by
\begin{align}
  \mathsf{Regret} &\leq
  4 \sum_{i = 1}^B n_{i} \gamma_{i - 1} \leq
  4 q (1 + \nu d \log \log d) \sum_{i = 0}^{B-1} q^{i} \gamma_{i},
  \label{eq:regret-bound-M2-start}
\end{align}
where $\gamma_i$ is the width of the confidence interval at round $i$.
The first term in the sum $\sum_{i} q^i \gamma_i$ is
\begin{align}
  \sum_{i = 0}^{B-1} q^{i}
  \sqrt{\frac{d \log(1 / \delta)}{\nu q^i}} + q^i \frac{\sqrt{d} \log(1/\delta)}{\nu q^i \epspriv}
  &\leq 
  \sqrt{\frac{d \log(1 / \delta)}{\nu}} \left( \sum_{i = 0}^{B-1} q^{i / 2}
    + \frac{\sqrt{\log(1 / \delta)}}{\epspriv \sqrt{\nu}} \sum_{i = 0}^{B-1}
  \right) \notag \\
  &\leq \sqrt{\frac{d \log(1 / \delta)}{\nu}} \left(\frac{q^{B / 2} - 1}{q^{1/2} - 1}
    + \sqrt{\frac{\log(1 / \delta)}{\nu}} \frac{B - 1}{\epspriv}
  \right),
  \label{eq:regret-M2-1}
\end{align}
which is a term independent of the corruption fraction. The second group
of summands in $\sum_{i} q^i \gamma_i$ is
\begin{align}
  & \sum_{i = 0}^{B-1} q^i \left( 1 + \sqrt{\frac{\log(d \log \log d / \delta)}{\nu q^i}}
    + \frac{\log(d \log \log d / \delta)}{\nu q^i \epspriv} \right) \notag \\
  & =
  \frac{T}{1 + \nu d \log \log d} +
  \sqrt{\frac{\log(d \log \log d / \delta)}{\nu}} \cdot \frac{q^{B/2} - 1}{q^{1/2} - 1} +
  \frac{\log(d \log\log d/\delta)}{\nu} \cdot \frac{B - 1}{\epspriv}.
  \label{eq:regret-M2-2}
\end{align}
Finally, summing over $i$ using the last term of the confidence interval yields
\begin{align}
  \sum_{i = 0}^{B-1} q^i \alpha = \frac{\alpha T}{1 + \nu d \log \log d}.
  \label{eq:regret-M2-3}
\end{align}
Putting everything together, we arrive at the claimed regret bound:
\begin{align}
  \mathsf{Regret} &\lesssim
  \begin{aligned}[t]
  & (1 + \nu d \log \log d) \left(
    \sqrt{\frac{d T \log(1 / \delta)}{\nu}} +
    \frac{\log(1 / \delta) \log(T) \sqrt{d}}{\epspriv \sqrt{\nu}}
  \right) \\
  & + 2d \left(\sqrt{\alpha d \log \log d} + \sqrt{\alpha \log(1 / \delta)}\right)
  \left(T + \sqrt{\frac{T d \log \log d / \delta)}{\nu}} + \frac{\log(d \log \log d / \delta)}{\nu} \frac{\log T}{\epspriv}\right) \\
  & + \alpha T.
  \end{aligned}
\end{align}
\end{proof}

\end{document}